\newif\ifisarxiv
\isarxivtrue
\documentclass[twoside,11pt]{article}
\ifisarxiv
\usepackage{sty/jmlr2e-arxiv}
\else
\usepackage{sty/jmlr2e}
\fi

\usepackage{amsmath}
\usepackage{amssymb}
\usepackage{mathrsfs}
\usepackage{color}
\usepackage{cancel}
\usepackage{wrapfig}
\usepackage{caption}
\usepackage{algorithm}
\usepackage{xfrac}
\usepackage{algorithmic}
\usepackage{commath}

\newcommand{\sqrtlh}[1]{{\sqrt{\white{|}\!\!\smash{\text{\fontsize{9}{9}\selectfont$\hat{l}_{#1}$}}}}}
\newcommand\mydots{\makebox[1em][c]{.\hfil.\hfil.}}

\def\Sd{\mathscr{S}_{\!d}}
\newcommand{\D}{\mathrm{D}}
\newcommand{\DPP}{\mathrm{DPP}}
\newcommand{\dx}{\D_{\!\cal X}} 
\newcommand{\dyx}{\D_{\!\cal Y|\x}} 
\newcommand{\dyxi}{\D_{\!\mathcal{Y}|\x_i}}
\newcommand{\dxk}{\D_{\!\cal X}^k} 
 
\newcommand{\dxy}{\D}
\newcommand{\dxyt}{\widetilde{\D}}
\newcommand{\dxyh}{\widehat{\D}}

\newcommand{\lev}{\mathrm{Lev}_{\dx}}
\newcommand{\levh}{\mathrm{Lev}_{\Sigmabh,{\cal X}}}
\newcommand{\levhh}{\mathrm{\widehat{L}ev}}
\newcommand{\dxt}{{\widetilde{\D}_{\cal X}}}
\newcommand{\dxh}{{\widehat{\D}_{\cal X}}}
\def\vskx{{\mathrm{VS}_{\!\dx}^k}}

\def\vskxm{{\mathrm{VS}_{\!\dx}^{k-1}}}

\def\vsdx{{\mathrm{VS}_{\!\dx}^d}}

\newcommand{\sigd}{\boldsymbol\Sigma_{\dx}}

\def\kd{K_{\!\dx}}
\def\poly{{\mathrm{poly}}}

\def\Vol{{\mathrm{VS}}}
\def\Lev{{\mathrm{Lev}}}

\newenvironment{proofof}[2]{\par\vspace{2mm}\noindent\textbf{Proof of {#1} {#2}}\ }{\hfill\BlackBox\\[2mm]}

\DeclareMathOperator{\sgn}{\textnormal{sgn}}
\DeclareMathOperator{\adj}{\textnormal{adj}}

\def\xib{\boldsymbol\xi}
\def\Sigmab{\mathbf{\Sigma}}
\def\Sigmabh{\widehat{\Sigmab}}

\def\Sigmabb{\bar{\Sigmab}}

\def\xbt{\widetilde{\x}}
\def\xbh{\widehat{\x}}
\def\xbb{\bar{\x}}
\def\rbb{\bar{\r}}
\def\rb{\bar{r}}

\def\val {{\mathrm{val}}}

\def\ktd{k^{\underline{d}}}

\def\Xt{\widetilde{X}}

\def\ybb{\bar{\y}}
\def\yb{\bar{y}}

\def\Ec{\mathcal{E}}
\def\Nc{\mathcal{N}}

\def\Hc{\mathcal{H}}

\def\r{\mathbf r}

\def\K{\mathbf K}

\ifx\BlackBox\undefined
\newcommand{\BlackBox}{\rule{1.5ex}{1.5ex}}  
\fi
\DeclareMathOperator*{\argmin}{\mathop{\mathrm{argmin}}}

\def\x{\mathbf x}
\def\y{\mathbf y}
\def\ybh{\widehat{\mathbf y}}
\def\ybt{\widetilde{\mathbf y}}
\def\yh{\widehat{y}}

\def\yt{\widetilde{y}}

\def\a{\mathbf a}
\def\b{\mathbf b}
\def\w{\mathbf w}
\def\v{\mathbf v}

\def\wbh{\widehat{\mathbf w}}
\def\wh{\widehat{\mathbf w}}

\def\wbt{\widetilde{\mathbf w}}
\def\e{\mathbf e}
\def\zero{\mathbf 0}
\def\one{\mathbf 1}
\def\u{\mathbf u}

\def\X{\mathbf X}

\def\B{\mathbf B}
\def\A{\mathbf A}
\def\C{\mathbf C}
\def\U{\mathbf U}

\def\V{\mathbf V}

\def\I{\mathbf I}

\def\A{\mathbf A}
\def\P{\mathbf P}

\def\Xt{\widetilde{\mathbf{X}}}
\def\Xb{\bar{\mathbf{X}}}
\def\Xh{\widehat{\mathbf{X}}}

\def\E{\mathbb E}
\def\R{\mathbb R} 
\def\Pr{\mathrm{Pr}} 
\def\tr{\mathrm{tr}}

\def\rank{\mathrm{rank}}

\def\Var{\mathrm{Var}}

\newcommand{\defeq}{\stackrel{\textit{\tiny{def}}}{=}}

\let\origtop\top
\renewcommand\top{{\scriptscriptstyle{\origtop}}} 

\definecolor{silver}{cmyk}{0,0,0,0.3}
\definecolor{yellow}{cmyk}{0,0,0.9,0.0}
\definecolor{reddishyellow}{cmyk}{0,0.22,1.0,0.0}
\definecolor{black}{cmyk}{0,0,0.0,1.0}
\definecolor{darkYellow}{cmyk}{0.2,0.4,1.0,0}
\definecolor{orange}{cmyk}{0.0,0.7,0.9,0}
\definecolor{darkSilver}{cmyk}{0,0,0,0.1}
\definecolor{grey}{cmyk}{0,0,0,0.5}
\definecolor{darkgreen}{cmyk}{0.6,0,0.8,0}

\newcommand{\white}[1]{{\textcolor{white}{#1}}}

\ifx\proof\undefined
\newenvironment{proof}{\par\noindent{\bf Proof\ }}{\hfill\BlackBox\\[2mm]}
\fi

\ifx\theorem\undefined
\newtheorem{theorem}{Theorem}
\fi

\ifx\example\undefined
\newtheorem{example}{Example}
\fi

\ifx\condition\undefined
\newtheorem{condition}{Condition}
\fi
\ifx\property\undefined
\newtheorem{property}{Property}
\fi

\ifx\lemma\undefined
\newtheorem{lemma}[theorem]{Lemma}
\fi

\ifx\proposition\undefined
\newtheorem{proposition}[theorem]{Proposition}
\fi

\ifx\remark\undefined
\newtheorem{remark}[theorem]{Remark}
\fi

\ifx\corollary\undefined
\newtheorem{corollary}[theorem]{Corollary}
\fi

\ifx\definition\undefined
\newtheorem{definition}{Definition}
\fi

\ifx\conjecture\undefined
\newtheorem{conjecture}[theorem]{Conjecture}
\fi

\ifx\axiom\undefined
\newtheorem{axiom}[theorem]{Axiom}
\fi

\ifx\claim\undefined
\newtheorem{claim}[theorem]{Claim}
\fi

\ifx\assumption\undefined
\newtheorem{assumption}[theorem]{Assumption}
\fi

\ifx\condition\undefined
\newtheorem{condition}[theorem]{Condition}
\fi

\numberwithin{equation}{section}
\numberwithin{figure}{section}
\numberwithin{theorem}{section}
\numberwithin{table}{section}

\usepackage{tikz}
\usetikzlibrary{positioning}

\newcommand{\bipgraph}[2]{%
    \begin{tikzpicture}[baseline=+.9
	ex,rotate=90,scale=.33,every node/.style={circle,draw,scale=.1}]
    \foreach \xitem in {1,2,...,#1}
    {%
    \node at (0,\xitem) (a\xitem) {};
    \node at (2,\xitem) (b\xitem) {};   
    }%

    \foreach \x [count=\xi] in {#2}
    {%
    \foreach \tritem in \x 
    \draw(a\xi) -- (b\tritem);
    }
    \end{tikzpicture}  
}

\ShortHeadings{Unbiased estimators for random design regression}{Derezi\'{n}ski, Warmuth and Hsu}
\firstpageno{1}

\begin{document}

\title{Unbiased estimators for random design regression}

\author{%
\name Micha{\l } Derezi\'{n}ski \email derezin@umich.edu\\
\addr Department of Electrical Engineering \& Computer Science,
University of Michigan
\AND 
\name Manfred K. Warmuth \email manfred@google.com\\
\addr 
UC Santa Cruz and Google Inc.
\AND
\name Daniel Hsu \email djhsu@cs.columbia.edu\\
\addr Department of Computer Science,
Columbia University
}
\editor{}

\maketitle

\begin{abstract}

    In linear regression we wish to estimate the optimum linear least squares
  predictor for a distribution over $d$-dimensional input
  points and real-valued responses, based on a small sample. Under standard random design
  analysis, where the sample is drawn i.i.d.~from the input distribution, the
  least squares solution for that sample can be viewed as the natural
  estimator of the optimum. Unfortunately, this estimator 
  almost always incurs an undesirable bias coming from the randomness
  of the input points, which is a significant bottleneck in model
  averaging. In this paper we show that it is
  possible to draw a non-i.i.d.~sample of input points such that,
  regardless of the response model, 
  the least squares solution is an unbiased estimator of the
  optimum. Moreover, this sample can be produced efficiently by augmenting a 
  previously drawn i.i.d.~sample with an additional set of $d$ points,
  drawn jointly according to a certain determinantal point process constructed from
  the input distribution rescaled by
  the squared volume spanned by the points. Motivated by this,
  we develop a theoretical framework for studying
  volume-rescaled sampling, and in the process prove a number of new matrix
  expectation identities. 
  We use them to show that for any input
  distribution and $\epsilon>0$ there is a random design consisting of
  $O(d\log d+ d/\epsilon)$ points from which an unbiased estimator can
  be constructed whose expected square loss over  
  the entire distribution is bounded by $1+\epsilon$ times
  the loss of the optimum. 

  We provide efficient algorithms for constructing such unbiased
  estimators in a number of practical settings.  In one such setting, we let the input
  distribution be uniform over a large dataset of $n\gg d$ points. Here,
we obtain the first unbiased least
  squares estimator that can be constructed in time nearly-linear in
  the data size, resulting in strong guarantees for model
  averaging. We achieve these computational gains by introducing a new
  algorithmic technique, called distortion-free intermediate 
  sampling, which is the first method to enable sampling from
  determinantal point processes in time polynomial in the
  sample size.
\end{abstract}


\section{Introduction}

We consider linear regression where the examples
$(\x^\top,y)\in\R^d\times \R$ are generated by an unknown distribution
$\D$ over $\R^d \times \R$, with $\dx$ denoting the marginal
distribution of a row vector $\x^\top$ and $\dyx$ denoting the conditional distribution of $y$ given $\x$.
In statistics, it is common to assume that the response $y$ is a linear function of $\x$ plus zero-mean Gaussian noise; the goal is then to estimate this linear function.
We decidedly make no such assumption.
Instead, we allow the distribution to be
arbitrary except for the nominal requirement that the
second moments of the point $\x$ and response $y$ are
bounded, i.e., $\E[\|\x\|^2] < \infty$ and $\E[y^2] < \infty$.
The target of the estimation is the linear least squares predictor of $y$ from $\x$ with respect to $\D$:
$$\w^*_{\dxy}\defeq
\argmin_{\w\in\R^d}L_{\dxy}(\w),
\quad\text{where}\;\; 
L_{\dxy}(\w)\defeq \E\big[(\x^\top\w-\y)^2\big] . $$
Here, we assume $\E[\x\x^\top]$ is invertible so we have
the concise formula $\w^*_{\dxy} = (\E[\x\x^\top])^{-1}\E[\x y]$.
Our goal is to construct a ``good'' estimator of this
target $\w^*_{\dxy}$ from a small sample. For the rest of the
paper we use $\w^*$ as a shorthand.

In our setup, the estimator $\wh$ of $\w^*$ is based on solving a
least squares problem on a sample of $k$ examples
$(\x_1^\top,y_1),\dotsc,(\x_k^\top,y_k)$. 
We assume that given $\x_1,\dotsc,\x_k$, the responses $y_1,\dotsc,y_k$ are conditionally independent, and 
the conditional distribution of $y_i$ only depends on
$\x_i$, i.e., $y_i \sim \dyxi$ for $i=1,\dotsc,k$.
However, for the applications we have in mind, the marginal distribution of $\x_1,\dotsc,\x_k$ is allowed to be flexibly designed based on $\dx$.
The most standard choice is i.i.d.~sampling from the distribution $\dx$ of $\x$, i.e., $(\x_1^\top,\dotsc,\x_k^\top) \sim \dxk$.
We shall seek other choices that can be implemented given the ability to sample from $\dx$ but that lead to better statistical properties for $\wh$.

In particular, the properties we want of the estimator $\wh$ are the following.
\begin{enumerate}
  \item Unbiasedness: \ $\E[\wh]=\w^*$.
  \item Near-optimal expected loss: \ $\E\big[L_{\dxy}(\wh)\big]\le
    (1+\epsilon) L_{\dxy}(\w^*)$ for some (small) $\epsilon>0$. 
  \end{enumerate}
Together, these properties have many useful implications, such as a
bound on the out-of-sample prediction variance, i.e., $\Var[\x^\top\wbh]\leq\epsilon$
for $\x^\top\sim\dx$, and improved guarantees for averaging, e.g.,
$\E\big[L_{\dxy}(\frac{\wbh_1+\wbh_2}2)\big]\leq(1+\frac\epsilon2)L_{\dxy}(\w^*)$,
where $\wbh_1$ and $\wbh_2$ are independent copies of $\wbh$.
The central question is how to sample $\x_1,\dotsc,\x_k$ to achieve
these properties with sample size $k = k(\epsilon)$ as small as possible.
Note that while in general it is very natural to seek an
\emph{unbiased} estimator, in the context of
random design regression it is highly unusual. This is because, as we
discuss shortly, standard approaches fail in this regard.
In fact,
until recently, unbiased estimators have been considered out of reach for this problem.

An important and motivating case of our general setup occurs when $\dx$ is
the uniform distribution over a fixed set of $n$ points and
$\dyx$ is deterministic. That is, there is an $n\times d$ {\em fixed design
matrix} $\X$ and a response vector $\y\in\R^n$ 
such that the distribution is uniform over the $n$ rows. Here, the loss of $\w$
can be written as $L_{\dxy}(\w) = \tfrac1n \|\X\w-\y\|^2$. This traditionally
fixed design setting turns into a random design when we are required to
sample $k\ll n$ rows of $\X$, observe \emph{only} the entries of $\y$
corresponding to those rows,
and then construct an estimate $\wh$ of the least squares solution for
all of $(\X,\y)$. Such constraints are imposed either in the context of
experimental design and active learning, where $k$ represents the
budget of responses that we are allowed to observe (e.g., because the
responses are expensive), or to reduce the
computational cost of solving the full least squares problem. Here, an important motivation for
unbiasedness is parallel and distributed model averaging, where we wish
to aggragate many independent copies of an estimator. See
Section~\ref{s:app} for further discussion of model averaging and
experimental design. 

Throughout the introduction we give some intuition about
our results by discussing the one dimensional case. 
For example, consider the following $2\times 1$ fixed design problem:
\begin{align}
  \X=\begin{bmatrix}\text{\fontsize{9}{9}\selectfont$x_1\!\!:$}\ 1\
    \\
    \text{\fontsize{9}{9}\selectfont$x_2\!\!:$}\ 2\ \end{bmatrix}
  ,\quad\y=\begin{bmatrix}
 \text{\fontsize{9}{9}\selectfont$y_1\!\!:$}\ 1\ \\
\text{\fontsize{9}{9}\selectfont$y_2\!\!:$}\ 1\ \end{bmatrix},\quad\text{with
  target:}\quad \w^*=\frac{\sum_ix_iy_i} {\sum_i x_i^2}= \frac{3}{5}.\label{eq:1-dim}
  \end{align}
Suppose that we wish to estimate the target after observing only
a single response (i.e., $k=1$). If we draw the response uniformly at
random (i.e., from the distribution $\dxy$), then the least
squares estimator for this sample will be a \emph{biased} estimate of the target:
$\E[\wh]=\frac{1}{2} \frac{y_1}{x_1}+ \frac{1}{2} \frac{y_2}{x_2} 
	=\frac{3}{4}\ne \frac{3}{5}.$

The bias in least squares estimators is present even when
each input component is drawn independently from a
standard Gaussian. As an example, we let $d=5$ and set:
\begin{align*}
\x^\top\! = (x_1,\dots,x_d)\overset{\text{i.i.d.}}{\sim} \Nc(0,1) ,\quad\ y =
  \xi(\x)\! +\! \epsilon,\quad\text{where }  \   \xi(\x) =
  \sum_{i=1}^d x_i + \frac{x_i^3}{3},\quad \epsilon \sim \Nc(0,1).
\end{align*}
The response $y$ is
a non-linear function $\xi(\x)$ plus independent white noise
$\epsilon$. Note that it is crucial that the response contains some
non-linearity, and it is something that one would
expect in real datasets. The response is cubic and was chosen so that
it is easy to solve algebraically for the optimum solution $\w^*=\argmin_\w L_{\dxy}(\w)$
(see Appendix~\ref{a:exact}).

\begin{wrapfigure}{r}{0.45\textwidth}
    \vspace{-3mm}
    \includegraphics[width=0.5\textwidth]{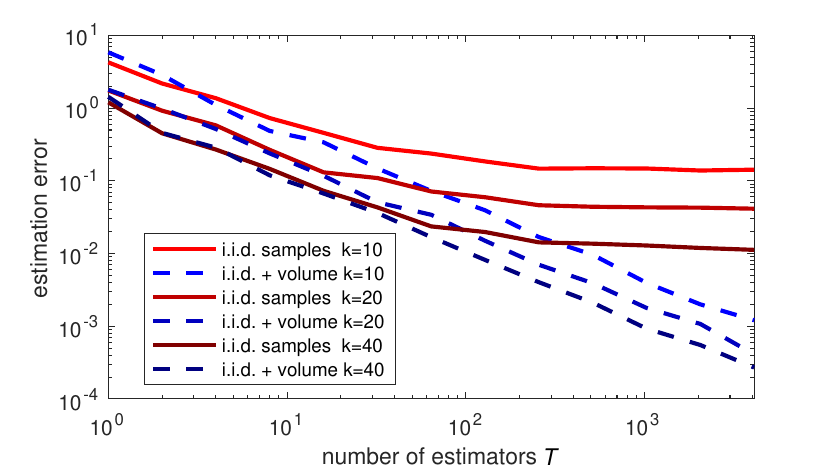}
    \vspace{-7mm}
  \caption{Averaging least squares estimators for Gaussian inputs with
    $d=5$.}\label{f:experiment}
  \vspace{3mm}
\end{wrapfigure}

For this Gaussian setup we evaluate the bias of the least squares estimator produced
for this problem by i.i.d.~sampling of $k$ points.
We do this by performing model averaging, i.e., producing many such estimators
$\wbh_1,\dots,\wbh_T$ independently,
and looking at the estimation error of the
average of those estimators $\wbt:=\frac1T\sum_{t=1}^T\wbh_t$:
\begin{align*}
    \text{estimation error:} \quad 
  \|\wbt-\w^*\|^2.
\end{align*}

Figure~\ref{f:experiment} (red curves) shows the experiment for several
values of $k$ and a range of values of $T$ (each presented data point is
an average over 50 runs). The i.i.d.~sampled estimator
is biased for any sample size (although the bias decreases with $k$),
and therefore the averaged estimator clearly does not converge to the optimum.
We next discuss how to construct an unbiased estimator (dashed blue
curves),
for which the estimation error of the averaged estimator
exhibits $\frac1T$ convergence to zero (regardless of $k$). 
This type of convergence appears as a straight line on the log-log
plot on Figure~\ref{f:experiment}.

Recently, \cite{unbiased-estimates-journal} developed the first method
for constructing \emph{unbiased} estimators in the case where $\dxy$
is uniform over a fixed design $(\X,\y)$. This method, which we will
refer to as
\emph{discrete volume sampling}, jointly draws a subset  
$S\subseteq [n]$ of $k$ rows of the design matrix $\X$ with
probability proportional to $\det(\X_S^\top\X_S)$, where $\X_S$
denotes the submatrix of $\X$ with rows indexed by $S$. For this
distribution, the linear least squares estimator 
$\wh=\X_S^\dagger\y_S$ is unbiased, i.e.,
$\E[\wh]=\w^*= \X^\dagger\y$, where $\X^\dagger$ denotes the
Moore-Penrose pseudoinverse.
Indeed, if we volume sample the set $S$ of size 1 in the example problem
\eqref{eq:1-dim}, then
$\E[\wh]=\frac{x_1^2}{\sum_ix_i^2}\,\frac{x_1y_1}{x_1^2}+\frac{x_2^2}{\sum_ix_i^2}\,\frac{x_2y_2}{x_2^2}
= \frac{\sum_ix_iy_i}{\sum_ix_i^2} = \w^*$.

\subsection{Our contributions}

\paragraph{Contribution 1: {\it Unbiased estimator for random design
  regression}}
Our first contribution in this paper is proposing a new unbiased
estimator for arbitrary distributions $\dxy$
(i.e., not just uniform over a fixed design matrix). Let the sample
$\x_1,\ldots,\x_k\in\R^d$ be drawn jointly with probability
proportional to $\det(\sum_{i=1}^k \x_i\x_i^\top)\, \dx^k (\x_1,\ldots,\x_k)$,
i.e., we reweigh the $k$-fold i.i.d.\ distribution $\dx^k$ by the
determinant of the sample covariance.
We refer to this as \emph{volume-rescaled sampling} from $\dxk$ and denote it
as $\vskx$. In this general context, we are able to prove that for
arbitrary distributions $\dx$ and $\dyx$, volume-rescaled sampling
produces unbiased linear least squares estimators (Theorem \ref{t:unbiased}). This result does
not follow from the fixed design analysis, and in obtaining it we
derive novel extensions of fundamental expectation identities for the
determinant of a random matrix. In the process, we develop a new
tool kit for computing expectations under volume-rescaled sampling,
which includes new expectation formulas for sampled 
pseudoinverses, inverses and adjugates.

\paragraph{Contribution 2: {\it Correcting the bias of
    i.i.d.~sampling}}
The fact that volume-rescaled sampling of size $k\ge d$ \emph{always} produces unbiased estimators
of the target $\w^*$ stands in contrast to i.i.d.\ sampling from $\dx$
which generally fails in this regard. Yet surprisingly, we show
that a volume-rescaled sample of any size $k\ge d$ is essentially 
composed of an i.i.d.\ sample of size $k-d$ from $\dx$ plus a volume-rescaled sample
of size $d$ (Theorem \ref{l:composition}). This means that the linear least squares estimator of such
\emph{composed} sample is also unbiased. Thus, as an immediate corollary of
Theorems \ref{l:composition} and \ref{t:unbiased} we reach the
following remarkable conclusion:

\emph{Even though i.i.d.~sampling typically results in a biased least
  squares estimator, adding a volume-rescaled sample 
of size $d$ to the i.i.d.~sample eliminates that bias altogether:}

\vspace{2mm}
\centerline{
  \fbox{
    \quad\quad\begin{minipage}{30em}
      \vspace{-2mm}
\begin{align*}
&\text{\underline{i.i.d.} sample}& &(\x_1^\top,y_1),\ldots,(\x_k^\top,y_k) \sim D^k\\
&\text{sol. for \underline{i.i.d.} sample}&&
                                             \wbh=\argmin_\w \sum_{i}(\x_i^\top\w-y_i)^2 \\[2mm]
\hline \\[-3mm]
&\text{\underline{volume}-rescaled sample}\quad& &\overset{\text{$d$ points}}{\xbb_{1}^\top,\dots,\xbb_{d}^\top}\ \sim\ \det\!
\text{\fontsize{9}{9}\selectfont$\begin{pmatrix}
    -\xbb_{1}^\top-\\
\dots\\
-\xbb_{d}^\top-
\end{pmatrix}^{\!\!\!2}$}\!
\cdot \dx^d\ \ \, \text{\scriptsize($d$ - input dimension)}\\
&\text{query responses}&&\yb_{i}\ \sim\ \dxy_{{\cal Y}|\xbb_{i}},\quad \forall_{i=1..d} \\
&\text{sol. for \underline{i.i.d} + \underline{volume}}&&
    \wbt=\argmin_\w \Big\{\sum_{i}(\x_i^\top\w-y_i)^2 +\sum_{i}(\xbb_i^\top\w-\yb_i)^2 \Big\}
\end{align*}
\hrule\vspace{4mm}
\centerline{\textit{Our result:}\quad $\E[\wbt] = \w^*$ even though typically $\E[\wbh] \neq \w^*$}
\vspace{2.5mm}
\end{minipage}\quad\quad
}}
\vspace{3mm}

Indeed, in the simple Gaussian experiment used for Figure~\ref{f:experiment},
the estimators produced from i.i.d.\ samples augmented with
a volume-rescaled sample of size $d$ (dashed blue curves) become unbiased (straight lines).
To get some intuition, let us show how the bias disappears
in the one-dimensional fixed design case where $\dx$ is a uniform
sample from $\{(x_1,y_1),\dots,(x_n,y_n)\}$.
In this case, reweighing the probability of just the first sampled
point by its square already results in an unbiased estimator.
Let $\wbh$ be the least squares estimator computed from 
$(x_{i_1},y_{i_1}),\dots,(x_{i_k},y_{i_k})$ with all indices sampled uniformly from
$[n]$. Now, suppose that we replace $i_1$ with $i_1'$ sampled
proportionally to $x_{i_1'}^2$, and denote the modified estimator
as $\wbt$. Due to symmetry, it makes no difference which index
we choose to replace, so
\begin{align*}
\E\big[\wbt\big] = \E\Big[\tfrac{x^2_{i_1}}{\sum_jx_j^2}\,\wbh\Big] =
  \frac1k\sum_{t=1}^k\E\Big[\tfrac{x^2_{i_t}}{\sum_jx_j^2}\,\wbh\Big]
  = \frac{\E[\frac1k(\sum_t x_{i_t}^2)\,\wbh]}{\sum_jx_j^2}.
\end{align*}
By definition of the least squares estimator,
$\E[\frac1k(\sum_t x_{i_t}^2)\,\wbh] =
\E[\frac1k\sum_tx_{i_t}y_{i_y}]=\sum_jx_jy_j$, from which it follows
that $\E[\wbt]=\w^*.$ This simple argument at once shows the
unbiasedness of $\wbt$ and the \emph{composition} property
discussed in the previous paragraph. In higher dimensions, the analysis
gets considerably more involved, but it follows a similar outline.

\paragraph{Contribution 3: {\it Near-optimal expected loss bound}}
Perhaps surprisingly, volume-rescaled sampling may not
lead to estimators with near-optimal loss guarantees: We show that for
any $k\geq d$ there are
distributions $D$ for which volume-rescaled sampling of size $k$
results in the linear least squares estimator
having loss at least twice as large as the optimum
loss (with probability at least 0.25).
However, we remedy this bad behavior by composing a volume-rescaled
sample of size $d$ with an i.i.d.~leverage score sample of size $k-d$. This composition
achieves the following feat: It does not affect the unbiasedness of
the estimator and, and it leads to good approximation properties.
Specifically, in Theorem \ref{t:loss} we show that
$k=O(d \log d + d/\epsilon)$ points are sufficient to construct an
estimator $\wbh$ such that:
\begin{align*}
  \E[\wbh]=\w^*\quad\text{and}\quad\E\big[L_{\dxy}(\wbh)\big]\leq(1+\epsilon) L_{\dxy}(\w^*).
\end{align*}
Note that an analogous loss bound is achievable for vanilla i.i.d.\ leverage score
sampling, but (1) the estimators produced from leverage score sampling
are biased, and (2) the expected loss bound holds only if we condition
on a certain high-probability event (both of those are
significant issues, e.g., in the context of model averaging). To show the
expected loss bound that holds without conditioning and for an unbiased estimator, we
break the analysis into 
two cases, depending on whether the high-probability event occurs. When
it does not, then our analysis crucially relies on the expectation formulas we 
develop for volume-rescaled sampling. Note that the only expected loss bound
previously developed for a volume-based sampling distribution
was limited to fixed design, and required $d^2/\epsilon$ points to obtain
an approximation factor of $1+\epsilon$ \citep{unbiased-estimates-journal}. To our
knowledge, that analysis does not easily extend to $k>d$, which is why
our techniques are radically different.

\paragraph{Contribution 4: {\it Accelerated sampling algorithms}}
Our work also leads to sampling algorithms which significantly improve
on the state-of-the-art time complexity of volume-rescaled sampling,
both in the fixed and random design settings, with further algorithmic
implications for the broader class of determinantal point processes
(see Section \ref{s:app-dpp}). We achieve this by
introducing a new technique called \emph{distortion-free intermediate sampling}:
We first sample a larger pool of points 
based on approximate i.i.d.\ leverage scores and then down-sample from
that pool to construct the volume-rescaled sample. 
We use rejection sampling for the down-sampling step to ensure
exactness of the resulting overall sampling
distribution. Surprisingly, this does not adversely affect the
complexity because of the provably high acceptance rate during
rejection sampling (see Theorem \ref{t:det}).

When distribution $\dxy$ is defined by a fixed design $(\X,\y)$ with $n\gg
d$ data points, then, in Theorem \ref{t:finite}, we improve
upon the time complexity of discrete volume sampling from $O(nd^2)$ to
$O(nd\log n + d^4\log d)$. This cost is nearly-linear in the size of the
dataset and, for the first time, better than solving the full least
squares problem directly, which takes $O(nd^2)$ time.
Importantly, most of the cost 
in the new algorithm  comes from
preprocessing, and the actual sampling takes only $O(d^4)$ time,
i.e., independent of the data size,  which is useful when we wish to
produce multiple independent samples.  Combining this with the new
loss bound, we get the following improvements for obtaining an 
unbiased subsampled estimator with loss within $1+\epsilon$ of the optimum:
The sample size $k$ is reduced from $O(d^2/\epsilon)$ to
$O(d\log d+d/\epsilon)$ and the time complexity from $O(nd^3/\epsilon)$ to 
$O(nd\log n + d^4\log d+d^3/\epsilon)$. 

Remarkably, we show that \emph{exact} volume-rescaled sampling is
possible even when distribution $\dx$ is unknown (and possibly
continuous) and we only have oracle access to it.
In this setting, the size of the intermediate sample that is necessary to achieve this grows 
linearly with a certain condition number of the
distribution (this is likely unavoidable in general). Finally,
in the special case where $\dx$ is a multivariate Gaussian distribution with unknown
covariance, we use a different approach to show that only $d+2$
additional samples from $\dx$ are needed to 
modify a sample from $\dxk$ so that it becomes a volume-rescaled
sample of size $k$.

  \subsection{Applications of our results}
  \label{s:app}
While studying unbiased estimators for least squares regression is an
old and classical problem, our new results have significant
implications for modern data science, both from a computational and
statistical perspective. We outline these implications below, along
with some of the recent related work.

\subsubsection{Model averaging}
Model averaging is a standard technique for boosting the accuracy of a
subsampled estimator 
by constructing multiple independent copies and then averaging
them. This is particularly effective in parallel and distributed
environments, where the computational cost of constructing multiple
estimators is the same as the cost of computing one estimator. While
model averaging has been proposed as a strategy for least squares
regression \citep[e.g., see][]{wang2017sketched}, the bias which
arises for commonly used estimators (e.g., 
based on i.i.d.~sampling) constitutes a significant bottleneck for
this approach.

Our framework for constructing unbiased estimators with
expected loss bounds is uniquely suited for addressing the
problem of estimation bias in model averaging. To see this, consider a
least squares estimator $\wbh$ that 
satisfies both the unbiasedness property, $\E[\wbh]=\w^*$, and
near-optimal expected loss,
$\E[L_{\dxy}(\wbh)]\leq(1+\epsilon)L_{\dxy}(\w^*)$. It immediately
follows that if we construct $m$ independent copies
$\wbh_1,...,\wbh_m$ of $\wbh$, then the averaged estimator satisfies:
\begin{align*}
  \E\big[L_{\dxy}(\wbt)\big]\leq
  \Big(1+\frac\epsilon m\Big)\,L_{\dxy}(\w^*),\quad\text{where}\quad
  \wbt=\frac1m\sum_{i=1}^m\wbh_i.
\end{align*}
Consider for instance the setting where distribution $\dxy$ is defined
by a fixed  design $(\X,\y)$ with $n$ data points. Here, we can use
parallel averaging to boost the accuracy of a subsampled least squares
estimator from $\epsilon$ to $\epsilon/m$ at virtually no additional
computational cost. However, for this to be practical, (1) the
estimator must be unbiased, and (2) the computational cost of
constructing the estimator must be less than $O(nd^2)$, the cost of
solving least squares exactly. We develop the first such estimator,
by not only providing an improved expected loss bound for an unbiased
estimator, but also reducing the  computational cost to
$O(nd\log n+d^4\log d)$, which is much less than $O(nd^2)$ when $n$
is sufficiently larger than $d$. Finally, we point out that our
volume-based sampling algorithms for model averaging have recently proven relevant
in the context of model averaging for distributed second-order optimization and
distributed ridge regression, among others
\citep{surrogate-sketching}.

\subsubsection{Experimental design}

A natural application for volume-rescaled sampling algorithms comes in
the context of experimental design \citep[a.k.a.~optimal design of
experiments; see][]{optimal-design-book,pukelsheim2006optimal}. Here, the goal is to select a
small set of data points for which 
the least squares estimator minimizes a given optimality criterion,
typically related to some notion of variance. Classical
experimental design imposes statistical assumptions on the response
model, making the least squares estimator trivially unbiased
regardless of how we select the set of points. Volume-rescaled
sampling provides a way of preserving the unbiasedness property while
relaxing the assumptions on the responses. In particular, this leads
to a fundamental connection between the expected loss and the
prediction variance, a standard optimality criterion (V-optimality) in experimental
design. Namely, for an estimator $\wbh$ such that $\E[\wbh]=\w^*$,
letting $\x^\top\sim\dx$, we have:
\begin{align*}
  \underbrace{\E\big[L_{\dxy}(\wbh)\big]-L_{\dxy}(\w^*)}_{\text{Excess
  loss}}\hspace{2mm} =
  \hspace{-2mm}\underbrace{\Var[\x^\top\wbh]}_{\text{Prediction variance}}\hspace{-4.5mm}.
\end{align*}
In a recent follow-up work, \cite{minimax-experimental-design} used
these ideas to develop a general framework for experimental design,
which bridges the gap between the statistical perspective (linear
response model) and the setting studied here (arbitrary responses),
relying on our volume-rescaled sampling tool kit (in particular,
Theorem~\ref{l:composition}). Furthermore, our
strategy of combining volume-based sampling methods with
i.i.d.~importance sampling (e.g., leverage scores) has
proven  instrumental in developing randomized rounding methods for
efficiently solving a range of experimental design problems (including
A/C/D/V-optimal design, and Bayesian experimental design), 
drastically reducing their computational cost and improving the
approximation quality, both for discrete
\citep{proportional-volume-sampling,bayesian-experimental-design} and
continuous domains \citep{poinas2020proportional}. 

\subsubsection{Determinantal point processes}
\label{s:app-dpp}

Volume-rescaled sampling of size $d$ (i.e., $\vsdx$, see
Definition~\ref{d:vol}) belongs to a family of distributions called
Determinantal Point Processes (DPPs), which has been studied
extensively in many computational areas as a tractable model of diverse
sampling, including in randomized numerical linear algebra
\citep{dpps-in-randnla}, machine learning \citep{dpp} and
statistics \citep{dpp-statistics}; here we cite selected surveys that
provide a thorough literature review. Our results lead to direct improvements in the
computational cost of sampling for an important class of so-called
Projection DPPs. We outline this here for the case
where the support of the distribution is a finite set.

Determinantal point processes are most commonly defined as a distribution over
subsets $S\subseteq \{1,...,n\}$,
parameterized by a positive semidefinite $n\times n$ kernel matrix $\K$ with all
eigenvalues in $[0,1]$, so that a sample $S\sim \DPP(\K)$ satisfies:
\begin{align*}
  \Pr(T\subseteq S) = \det(\K_{T,T}),\quad\text{for all}\quad
  T\subseteq \{1,...,n\}.
\end{align*}
Here, $\K_{T,T}$ denotes the $|T|\times |T|$ submatrix of $\K$ indexed
by $T$. When $\K$ is a projection matrix, i.e., all of its eigenvalues
are in $\{0,1\}$, then this is called a Projection DPP
and the size of the sampled set $S$ is equal to the rank of $\K$. An
alternate parameterization of a Projection DPP that appears in the
literature relies on an $n\times
d$ matrix $\X$ such that the kernel $\K=\X\X^\dagger$ is the rank $d$ projection
onto the column span of $\X$. By letting $\dx$ be
uniform over the rows of $\X$, we obtain that $\vsdx$ is the
distribution of $\X_S$ for $S\sim\DPP(\K)$, up to a permutation of the
rows (here, $\X_S$ indicates the rows of $\X$ indexed by $S$).

Prior to our work, the cost of generating each sample from a given
Projection DPP was $O(nd^2)$, both for the $\X$ and the $\K$
parameterizations, by using the algorithm of
\cite{dpp-independence}. Our technique of \emph{distortion-free
  intermediate sampling} drastically reduces these costs when $n\gg
d$. If we are using the $\X$ parameterization, then after an initial
preprocessing cost of $O(nd\log n + d^4\log d)$, we can sample from a
Projection DPP in time $O(d^4)$. When given an $n\times n$ projection
matrix $\K$ of rank $d$, we can sample from $\DPP(\K)$ in time
$O(d^6)$. Here, the preprocessing step involves simply reading the
diagonal of $\K$ in $O(n)$ time. In both cases, these are the
first $\poly(d)$ time sampling algorithms for Projection DPPs. Follow-up works
\citep{dpp-intermediate,dpp-sublinear,alpha-dpp} have extended
distortion-free intermediate sampling to the class of L-ensemble DPPs,
and more recently even beyond DPPs, to larger distribution families
such as strongly Rayleigh measures, which have many applications in machine
learning and theoretical computer science \citep{isotropy-log-concave,domain-sparsification}.

\subsection{Related work}
A discrete variant of volume-rescaled sampling of size $k= d$ was introduced to computer
science literature by \cite{pca-volume-sampling} for sampling from a
finite set of $n$ vectors, with algorithms
given later by
\cite{efficient-volume-sampling,more-efficient-volume-sampling}. A first
extension to samples of size $k>d$ is due to \cite{avron-boutsidis13},
with algorithms by
\cite{dual-volume-sampling,unbiased-estimates-journal,leveraged-volume-sampling},
and additional applications in experimental design explored by
\cite{tractable-experimental-design,proportional-volume-sampling,symmetric-polynomials}.
Prior to this work, the best known time complexity for this sampling
method, called here \emph{discrete volume sampling}, was $O(nd^2)$, as
shown by \cite{unbiased-estimates-journal}. Here,
we give an $O(nd\log n+d^4\log d)$ time algorithm.

As discussed in Section \ref{s:app-dpp}, volume-rescaled sampling of
size $d$ is also known in the literature as a type of
determinantal point process, called Projection DPP \cite[to learn
more, see][]{dpps-in-randnla}. Projection DPPs
arise in many computational tasks outside of linear regression, such
as dimensionality reduction \citep{dpp-css},
numerical integration \citep{dpp-mcmc} and graph algorithms
\citep{guenoche1983random}, therefore, efficient sampling algorithms for these
distributions are of significant interest \citep{dpp-zonotope}. More
broadly, determinantal point processes have found machine learning
applications in recommendation systems \citep[e.g.,][]{dpp-shopping}, data
summarization \citep[e.g.,][]{dpp-video}, stochastic
optimization \citep[e.g.,][]{dpp-minibatch,randomized-newton}, and many others
\citep[see][]{dpp}. The algorithmic technique of distortion-free
intermediate sampling, introduced in this work, has already been
applied beyond Projection DPPs \citep{dpp-sublinear,alpha-dpp}, which makes
it relevant to all of these applications.

The unbiasedness of least squares estimators under volume-based
distributions was first explored in the context of sampling from finite
datasets by \cite{unbiased-estimates-journal}, drawing on observations
of \cite{bental-teboulle}. Focusing on small sample sizes,
\cite{unbiased-estimates-journal} proved multiplicative bounds for the
expected loss under sample size $k=d$ with discrete volume
sampling. Because the produced estimators are unbiased,
averaging $d/\epsilon$ such estimators results in an 
unbiased estimator based on a sample of size $k =
d^2/\epsilon$ with expected loss at most $1+\epsilon$ times the
optimum at a total sampling cost of $O(nd^2\cdot d/\epsilon)$. 
In contrast, our new techniques achieve an unbiased estimator with
sample size $O(d\log d+d/\epsilon)$ and time complexity 
$O(nd\log n + d^4\log d+d^3/\epsilon)$.
\cite{unbiased-estimates-journal} also showed additional variance bounds for discrete volume
sampling under the assumption that the responses are linear functions of the input
points plus white noise.  We extend them here to arbitrary
volume-rescaled sampling w.r.t.\ a distribution.

Other techniques applicable to our linear regression
problem include leverage score
sampling~\citep{drineas2006sampling} and algorithms based on spectral
sparsification~\citep[e.g.,][]{chen2017condition,kacham2020optimal}.
Leverage score sampling is an i.i.d.\ sampling procedure which achieves
loss bounds nearly matching the ones we obtain here for volume-rescaled
sampling, however it produces biased estimators 
and experimental results (see Section \ref{s:experiments}) show that
it has weaker performance for small sample sizes. 
A different and more elaborate sampling technique based on spectral
sparsification~\citep{batson2012twice,lee2015constructing} 
was recently shown to be effective for linear
regression~\citep{chen2017condition}: They
show that $O(d/\epsilon)$ samples suffice to produce
an estimator with expected loss $(1+\epsilon) L_{\dxy}(\w^*)$.
However this method also does not
produce unbiased estimators, which is a primary concern of this paper
and desirable in many settings, as discussed in Section \ref{s:app}.

\paragraph{Conference versions of this paper}
Our work greatly expands and generalizes the results
of two conference papers:
\cite{leveraged-volume-sampling,correcting-bias}.
The first paper introduced the leverage score rescaling method in the limited
context of discrete volume sampling, developed the new
intermediate sampling algorithm, and proved the 
$O(d \log d +d/\epsilon)$ sample size bound for obtaining
an unbiased estimator with a $(1+\epsilon)$ loss bound. Note that the
original loss bound was shown to hold with a constant probability, as
opposed to in expectation, which is a significant obstacle to
using it in the context of model averaging. The second paper showed
how to correct the bias of i.i.d.\ 
sampling using a small size $d$ volume-rescaled sample and refined
the analysis of intermediate sampling. The current
paper strengthens the loss bound of the first
conference paper to the desired in-expectation form (this requires new
technical tools such as Lemma \ref{l:decorrelation}), and generalizes
it to the case of an arbitrary data distribution $D$ (Theorem~\ref{t:loss}).
In the process, we develop new formulas for
the expectation of the inverses and pseudoinverses of random matrices
under volume-rescaled sampling (Theorems \ref{t:pseudoinverse} and
\ref{t:square-inverse}) and characterize the marginals of this
distribution (Theorem \ref{t:marginal}). We also extend the
decomposition property of volume-rescaled sampling given in the second
conference paper (Theorem \ref{l:composition}), 
thereby greatly simplifying our proofs.
Finally, we give a new lower bound that complements our main results
(Theorem \ref{t:bias}).

\subsection*{Outline}
In Section \ref{s:vol} we give our basic definition of
volume-rescaled sampling w.r.t.\ an arbitrary distribution over the
examples and prove the basic expectation formulas as well as the fundamental 
decomposition property which is repeatedly used in later sections.
We also show that the linear least squares estimator is unbiased under
volume-rescaled sampling.
The decomposition property is then used in Section \ref{s:loss}
to show that volume-rescaled leverage score sampling produces
a linear least squares estimator with loss at most $(1+\epsilon)L_{\dxy}(\w^*)$ for
sample size $O(d \log d+d/\epsilon)$. 
The lower bounds in Section \ref{s:lower} show that
i.i.d.~sampling leads to biased estimators and plain volume-rescaled
sampling does not have $1+\epsilon$ loss bounds.

In Section \ref{s:algs} we show that if $\dx$ is normal,
then $d+2$ additional samples can be used to construct a
volume-rescaled sample of size $k$. 
When the distribution $\dx$ is
arbitrary but we are given an approximation of the covariance
matrix of $\dx$, then a special variant of approximate leverage score
sampling can be used to construct a
larger intermediate sample that contains a volume-rescaled sample with high
probability. We then show how to construct an approximate
covariance matrix from additional samples from $\dx$. 
The number of samples we need grows linearly with a variant of a
condition number of $\dx$.
Finally we show how the new intermediate sampling
method introduced here
leads to improved time bounds in the fixed design case.

In Section \ref{s:experiments} we compare the performance of the algorithms 
discussed in this paper on some real datasets.
We conclude with an overview and some open problems in Section \ref{s:conclusion}.

\section{Volume-rescaled sampling}
\label{s:vol}
In this section, we formally define volume-rescaled sampling and
describe its basic properties. We then use it to introduce the central concept of
this paper: an unbiased estimator for random design least squares
regression.

\textbf{Notation.} \  Let $\a_i^\top$ denote the $i$th row of a matrix
$\A$, and let $\A_S$ be the submatrix of $\A$ containing rows of $\A$
indexed by the set $S$. Also, we use $\A_{-i}$, $\A_{\,:,-j}$ and $\A_{-i,-j}$ to
denote matrix $\A$ with $i$th row removed, $j$th column removed, and
both removed, respectively. When $\A$ is $d\times d$, we use $\adj(\A)$ to denote the
adjugate of $\A$ which is a $d\times d$ matrix such that
$\adj(\A)_{ij}=(-1)^{i+j}\det(\A_{-j,-i})$.
We use $\dx$ to denote the distribution of a $d$-variate
random row vector $\x^\top$ and we assume throughout that $\sigd=\E[\x\x^\top]$
exists and is full rank. Distribution $D$ is called $(d,1)$-variate if it
produces a joint sample $(\x^\top,y)$ where $\x\in\R^d$ and $y\in\R$.
A random $k\times d$ matrix consisting of $k$ independent rows distributed as
$\dx$ is denoted $\X\sim\dxk$. We also use the following standard shorthand:
  $\ktd = \frac{k!}{(k-d)!} = k\,(k-1)\dotsm(k-d+1)$.
\begin{definition}\label{d:vol}
    Given a $d$-variate distribution $\dx$ and any $k\geq d$, we define
  volume-rescaled size $k$ sampling from $\dx$ as a $k\times
  d$-variate probability measure $\vskx$ such that for any event $A\subseteq
  \R^{k\times d}$ measurable w.r.t.~$\dxk$, its probability is 
    \vspace{-1mm}
  \begin{align*}
    \vskx(A) \ &\defeq\  
	  \frac{\E\big[ \det(\X^\top\X) \cdot\one_{ [\X\in A]} \big]} 
	  {\E\big[\det(\X^\top\X)\big]},\quad\text{
                 where }\ \X\sim\dxk.
  \end{align*}
\end{definition}
For $k=d$, this volume-rescaled sampling is a type of Determinantal Point
Process known as Projection DPP \citep[see Section \ref{s:app-dpp}; to
learn more, see][]{dpps-in-randnla}. The case of $k>d$ can be  
viewed as an extension of that family of distributions.
\begin{remark}\label{r:expectations}
Distribution $\Xb\sim\vskx$ is well-defined whenever
  $\sigd=\E_{\dx}[\x\x^\top]$ is finite and full rank. Also, for any $F:\R^{k\times
    d}\!\rightarrow\!\R$, random variable $F(\Xb)$ is measurable if
  and only if $\det(\X^\top\X)F(\X)$ is measurable
  for $\X\sim\dxk$, and then it follows that
  \begin{align*}
    \E_{\Xb}[F(\Xb)] = \frac{\E_{\X}[\det(\X^\top\X)
    F(\X)]}{\E_{\X}[\det(\X^\top\X)]} =
    \frac{\E[\det(\X^\top\X) F(\X)]}{\ktd\,\det(\sigd)}.
  \end{align*}
  \end{remark}
The remark follows from a key lemma which is an
extension of a classic result by \cite{expected-generalized-variance},
who essentially showed \eqref{eq:l-det} below when $\A=\B$, but not \eqref{eq:l-adj}.
Part \eqref{eq:l-det} of the lemma lets us rewrite the normalization
of volume-rescaled sampling $\vskx$ as: 
\begin{align*}
    \E_{\X} \big[\det(\X^\top\X)\big] =
  (k^{\underline{d}}/k^d)\cdot\det\!\big(\E[\X^\top\X]\big) =
  k^{\underline{d}}\cdot\det\!\big(\sigd\big)
  ,\quad\text{ where }\Sigmab_{\dx}=\E_{\dx}[\x\x^\top].
\end{align*}

\begin{lemma}
  \label{l:determinant}
  If the rows of the random matrices $\A,\B\in\R^{k\times d}$
  are sampled as an i.i.d.~sequence of $k$ pairs of joint random vectors $(\a_i,\b_i)$, then
\begin{align}
  k^d\,\E \big[\det(\A^\top\B)\big]
  &= k^{\underline{d}}\,\det\!\big(\E[\A^\top\B]\big)
    &&\text{for any }k\geq d,\label{eq:l-det}\\
    k^{d-1}\,\E\big[\adj(\A^\top\B)\big]
  &=k^{\underline{d-1}}\,\adj\!\big(\E[\A^\top\B]\big)&&\text{for any
    }k\geq d-1.\label{eq:l-adj}
\end{align}
\end{lemma}
\begin{proof}
First, suppose that $k=d$, in which case $\det(\A^\top\B) =
\det(\A)\det(\B)$. 
Recall that by  definition the determinant can be written as:
  \begin{align*}
\det(\C) = \sum_{\sigma\in \Sd}\sgn(\sigma)\prod_{i=1}^dc_{i,\sigma_i},
\end{align*}
where $\Sd$ is the set of all permutations of $(1..d)$, and
$\sgn(\sigma) = \sgn\big((1..d),\sigma\big)\in\{-1,1\}$ is the parity
of the  number of swaps from $(1..d)$ to $\sigma$. Using this formula
and denoting  $c_{ij} = 
\big(\E[\A^\top\B]\big)_{ij}=d\,\E[a_{1i}b_{1j}]$, we can rewrite the expectation as: 
\begin{align*}
d^d\, \E\big[\! \det(\A)\det(\B)\big]\!
&=\!\!\!\sum_{\sigma,\sigma'\in \Sd}\!\!\!\sgn(\sigma)\sgn(\sigma')
\prod_{i=1}^d \E\big[d\cdot a_{i\sigma_i}b_{i\sigma'_i}\big]
\\
&= \sum_{\sigma\in \Sd}\sum_{\sigma'\in \Sd}\sgn(\sigma,\sigma') \prod_{i=1}^d c_{\sigma_i\sigma'_i}\\
  &=d! \sum_{\sigma'\in \Sd}\sgn(\sigma') \prod_{i=1}^d c_{i\sigma'_i}\\
  &=d!\det\!\big(\E[\A^\top\B]\big),
\end{align*}
which proves \eqref{eq:l-det} for $k=d$. The case of
$k>d$ follows by induction via a standard determinantal formula:
\begin{align*}
\E\big[\det(\A^\top\B)\big] &\overset{(*)}{=}
  \E\bigg[\frac1{k-d}\sum_{i=1}^k\det\!\big(\A_{-i}^\top\B_{-i}\big)\bigg]
  =\frac{k}{k-d}\,\E\big[\det\!\big(\A_{-k}^\top\B_{-k}\big)\big],
\end{align*}
where $(*)$ follows from the Cauchy-Binet formula. Finally, 
\eqref{eq:l-adj} can be derived  from \eqref{eq:l-det}:
\begin{align*}
  k^{d-1}\, \E\big[\adj(\A^\top\B)_{ij}\big]
  &= k^{d-1}\, \E\big[(-1)^{i+j}\det\!\big((\A^\top\B)_{-j,-i}\big)\big]\\
  &= (-1)^{i+j}\,
    k^{d-1}\E\big[\det(\A_{\,:,-j}^\top\B_{\,:,-i})\big]\\
\text{using \eqref{eq:l-det}}\quad  &= (-1)^{i+j}\,
    k^{\underline{d-1}}\det\!\big(\E[\A_{\,:,-j}^\top\B_{\,:,-i}]\big)\\
  &=k^{\underline{d-1}}\,(-1)^{i+j}\det\!\big((\E[\A^\top\B])_{-j,-i}\big)\\
  &=k^{\underline{d-1}}\,\adj\!\big(\E[\A^\top\B]\big)_{ij},
\end{align*}
where recall that $\A_{\,:,-j}\in\R^{k\times {d-1}}$ denotes matrix $\A$ with the $j$th column removed.
\end{proof}
\subsection{Basic properties}
In this section we look at the relationship between the random matrix
$\X\sim\dxk$ of an i.i.d.~sample from $\dx$ and the corresponding
volume-rescaled sample $\Xb\sim\vskx$. Even though the rows of $\Xb$
are not independent, we show that they contain among them an
i.i.d.~sample distributed according to $\dx^{k-d}$.
\begin{theorem}\label{l:composition}
Let $\Xb\sim\vskx$ and $S\subseteq[k]$ be a random size $d$ set
s.t.~$\Pr(S\,|\,\Xb)\propto \det(\Xb_S)^2$.
Then  $\Xb_S\sim\vsdx$, $\Xb_{[k]\backslash S}\sim \dx^{k-d}$, $S$ is
(marginally) uniformly random, and the three 
random variables $\Xb_S$, $\Xb_{[k]\backslash S}$, and $S$ are mutually independent.
\end{theorem}

Before proceeding with the proof, we would like to
discuss the implications of the theorem at a high level. 
First, observe that it allows us to
``compose'' a unique matrix $\Xb$ (which must be distributed according to $\vskx$)
from a $d$-row draw from $\vsdx$,
a $(k\!-\!d)$-row draw from $\dx^{k-d}$, 
and a uniformly drawn subset $S$ of size $d$ from $[k]$.
We construct $\Xb$ by placing the $d$ rows at row indices $S$ and
the $k-d$ rows at the remaining indices. Another way to think of the
construction of $\Xb$ is that we index the rows of $\vsdx$ from $1$ to
$d$ and the rows of $\dx^{k-d}$ from $d+1$ to $k$, and then permute
the indices by a random permutation $\sigma$:
{\large
\begin{align}
\text{volume + i.i.d.} \quad 
    &\overbrace{\x_1\dots\x_d}^{\vsdx}\,\,
                   \overbrace{\x_{d+1}\dots\dots\dots\dots\x_{k}}^{\dx^{k-d}}
		   \label{eq:volplusiid}
\\ \Updownarrow \quad\qquad 
    & \;\, \bipgraph{16}{2,6,7,10,4,15,16,12,3,11,8,14,13,9,1,5}
    \nonumber
\\  
    \vskx\qquad 
    & \; \x_{\sigma_1}\!\dots\dots\dots\dots\dots\dots\dots.\,.\,\x_{\sigma_{k}}
    \label{eq:volk}
\end{align}

Perhaps more surprisingly, given a volume-rescaled sample of size $k$ from $\dx$
(i.e., $\Xb\sim\vskx$),
sampling a set $S\subseteq[k]$ of size $d$ with probability $\propto \det(\Xb_S)^2$ (discrete
volume sampling) ``filters out'' a size $d$ volume-rescaled sample
from $\dx$
(i.e., $\Xb_S\sim\vsdx$). That sample is \emph{independent} of the remaining
rows in $\Xb$, so after reordering we recover \eqref{eq:volplusiid}.

We can repeat the steps of going ``back and forth'' between \eqref{eq:volplusiid} and
\eqref{eq:volk}. That is, we can compose a sample from $\vskx$ by
appending the size $d$ sub-sample we filtered out from $\Xb$ with its 
complement and permuting randomly,
and then again filter out a size $d$ volume sub-sample w.r.t.\ $\dx$
from the permuted sample. 
The size $d$ sub-samples produced the
first and second time are likely going to be different, but
they have the same distribution $\vsdx$.

This phenomenon can already be
observed in one dimension (i.e., $d=1$). 
In this case, \eqref{eq:volplusiid} samples one
point $x_1\sim x^2\cdot\dx$ and independently draws
$x_2,\dots,x_k\sim\dx^{k-1}$. Note that the $k$ random variables are mutually
independent but \emph{not} identically distributed. Now, if we
randomly permute the order of the variables as in \eqref{eq:volk},
then the new variables are identically distributed but \emph{not}
mutually independent. Intuitively, this is because
observing (the length of) any one of the
variables alters our belief about where the volume-rescaled sample was
placed. Applying Theorem \ref{l:composition}, we can now
``decompose'' the dependencies by sampling a singleton subset $S=\{i\}$ with probability
proportional to $x_i^2$. Even though the selected variable may not be
the same as the one chosen originally, it is distributed according to
volume-rescaled sampling w.r.t.\ $\dx$ and the remaining $k\!-\!1$ points
are i.i.d.\ samples from~$\dx$. 

\vspace{2mm}
\begin{proof}
The distribution of $S$ conditioned on $\Xb$ is the discrete volume
sampling distribution over sets of size $d$ whose normalization
constant is $\det(\Xb^\top\Xb)$ via the Cauchy-Binet formula. 
Denote $S^c=[k]\backslash S$ and let $A$, $B$ and $C$ be measurable events
for variables $S$, $\Xb_S$ and $\Xb_{S^c}$, respectively. We next show that
the three events are mutually independent and we compute their probabilities.
The law of total probability with respect to the joint distribution of
$S$ and $\Xb$, combined with Remark \ref{r:expectations} (using
$\X\sim\dxk$) implies that: 
\begin{align*}
  \Pr\big(S\!\in\! A\wedge\Xb_S\!\in\! B \wedge \Xb_{S^c}\!\in\!C\big)
  &=\E_{\Xb}\big[\,\Pr(S\!\in\! A\wedge\Xb_S\!\in\! B \wedge
    \Xb_{S^c}\!\in\!C\ |\ \Xb)\,\big]\\
  &=\frac{\E_{\X}\big[\det(\X^\top\X)\cdot
    \Pr(S\!\in\! A\wedge\Xb_S\!\in\! B \wedge \Xb_{S^c}\!\in\!C\ |\ \Xb\!=\!\X)
    \big]}{\ktd\det(\sigd)}\\
  & \overset{(a)}{=}\frac{\E\big[\det(\X^\top\X)\cdot\sum_{S\in A}
    \frac{\det(\X_S)^2}{\det(\X^\top\X)}\;\one_{[\X_S\in B]}\one_{[\X_{S^c}\in C]}
    \big]}{\ktd\det(\sigd)}\\
  &\overset{(b)}{=}\frac{\sum_{S\in A}\E\big[\det(\X_S)^2 \;\one_{[\X_S\in B]}
\;\one_{[\X_{S^c}\in C]}\big]}{\ktd\det(\sigd)}\\
  &\overset{(c)}{=}\frac{|A|\cdot\E\big[\det(\X_{[d]})^2 \;\one_{[\X_{[d]}\in
    B]}\big]\cdot \E\big[\one_{[\X_{[k]\backslash[d]}\in
    C]}\big]}{{k\choose d}d!\det(\sigd)}\\
  &=\frac{|A|}{{k\choose d}}\cdot\vsdx(B)\cdot \dx^{k-d}(C) .
\end{align*}
Here $(a)$ uses Cauchy-Binet to obtain the normalization
for $\Pr(S\,|\,\Xb)$, which is then cancelled out in $(b)$.
Finally $(c)$ follows because the rows of $\X\sim\dxk$ are i.i.d.~so $\X_S$ and
$\X_{S^c}$ are independent for any fixed $S$, and the choice of $S$
does not affect the expectation.
\end{proof}

Theorem \ref{l:composition} implies that for $k\gg d$, the distributions $\vskx$
and $\dxk$ are in fact very close to each other because they only
differ on a small sample of size $d$. 
Since the rows of $\Xb$ 
are exchangeable, they are also identically
distributed. The marginal distribution of a single row exhibits a key connection
between volume-rescaled sampling and leverage score
sampling (when generalized to our distribution setting), which we will exploit later.
Recall that for a fixed matrix $\X\in\R^{n\times d}$, the leverage
score of row $\x_i^\top$ is defined as $\x_i^\top (\X^\top\X)^{-1}\x_i$. 
Note that in this case, the $n$ leverage scores sum to $d$.
The following definition is a natural generalization of leverage scores 
to arbitrary distributions.
\begin{definition}\label{t:lev}
Given a $d$-variate distribution $\dx$, we define
leverage score sampling from $\dx$ as a $d$-variate probability measure 
$\lev$ such that for any event $A\subseteq\R^{1\times d}$ 
measurable w.r.t.~$\dx$, its probability is 
    \vspace{-1mm}
  \begin{align*}
    \lev(A) \ &\defeq\  
      \frac{\E_{\dx}\big[\one_{[\x^\top\in A]}\cdot\x^\top\sigd^{-1}\x\big]}
	   {\E_{\dx}[\x^\top\sigd^{-1}\x]},
	      \text{ where }\ \x^\top\sim\dx.
  \end{align*}
\end{definition}
Clearly, $\E_{\dx}[\x^\top\sigd^{-1}\x] =d$
when $\sigd$ finite. 
\begin{remark}\label{r:lev}
Distribution $\xbb\sim\lev$ is well-defined whenever
  $\sigd=\E[\x\x^\top]$ is finite and full rank. Also, for any $F:\R^{1\times
    d}\!\rightarrow\!\R$, random variable $F(\xbb^\top)$ is measurable if
  and only if $F(\xbb^\top)\,\xbb^\top \sigd^{-1}\xbb$ is measurable
  for $\x^\top\sim\dx$, and then it follows that
  \begin{align*}
    \E_{\lev}[F(\xbb^\top)] = \E_{\dx}[F(\x^\top)\,  \x^\top\sigd^{-1}\x]\,/d.
  \end{align*}
  \end{remark}
  
\begin{theorem}\label{t:marginal}
The marginal distribution of each row vector $\xbb_i^\top$ of
$\Xb\sim\vskx$ is 
\begin{align*}
  \frac dk \cdot\lev+ \Big(1-\frac dk\Big)\cdot \dx.
\end{align*}
\end{theorem}

\begin{proof}
For $k=d$, this can be derived from existing work on determinantal
point processes (see Lemma \ref{l:joint} for more details). We present
an independent proof using the identity $\det(\B+\v\v^\top)
= \det(\B) + \v^\top\!\adj(\B)\v$ and Lemma \ref{l:determinant}. Given $\Xb\sim\vsdx$,
\begin{align*}
  \Pr(\xbb_i^\top\in A)
   &= \frac{\E \big[\E[\one_{[\x_i^\top\in A]}
    \det(\X^\top\X)\,|\,\x_i]\big]}
{d!\det(\sigd)}\quad\quad\text{(where  $\X\sim\dx^d$)}\\
  &= \frac{\E\big[\one_{[\x_i^\top\in A]}
    \E[\det(\X_{-i}^\top\X_{-i}+\x_i\x_i^\top)\,|\,\x_i]\big]}
{d!\det(\sigd)}\\
  &\overset{(a)}{=}\frac{\E\big[\one_{[\x_i^\top\in A]}
    \E[\x_i^\top\!\adj(\X_{-i}^\top\X_{-i})\x_i\,|\,\x_i]\big]}
    {d!\det(\sigd)}\\
  &\overset{(b)}{=}\frac{\E\big[\one_{[\x_i^\top\in A]}\cdot\x_i^\top\!\adj(\sigd)\x_i\big]}
    {\frac{d!}{(d-1)!}\det(\sigd)}\\
  &\overset{(c)}{=}\E\big[\one_{[\x_i^\top\in A]}\cdot\x_i^\top\!\sigd^{-1}\x_i\big]\,/d.
\end{align*}
Here $(a)$ follows because $\det(\X_{-i}^\top\X_{-i})=0$,
and in $(b)$ we use Lemma \ref{l:determinant} and the fact that
$\E[\X_{-i}^\top\X_{-i}]=(d\!-\!1)\cdot\sigd$. 
    Finally $(c)$ employs the identity
$\adj(\A)=\det(\A)\A^{-1}$ which holds for any full rank $\A$. 
    The case of $k>d$ now follows from the case $k=d$ combined with Theorem \ref{l:composition}.
\end{proof}
The key random matrix that arises in the context of volume-rescaled
sampling is not $\Xb$ itself but rather its Moore-Penrose
pseudoinverse, $\Xb^\dagger = (\Xb^\top\Xb)^{-1}\Xb^\top$. Its expected value is given below.
\begin{theorem}\label{t:pseudoinverse}
  Let $\X\sim\dxk$ and
  $\Xb\sim\vskx$ for any $d$-variate $\dx$ and $k\geq d$. Then
  \begin{align*}
    \E\big[\Xb^\dagger\big] = \big(\E[\X^\top\X]\big)^{-1}\E[\X]^\top.
  \end{align*}
\end{theorem}
Recall that we assume $\E[\X^\top\X]=k\,\sigd$ is full rank throughout the paper.
The proof of Theorem \ref{t:pseudoinverse} is delayed to Section
\ref{s:unbiased} where we give a slightly more general statement
(Theorem \ref{t:unbiased}). We can compute not only the first moment
of $\Xb^\dagger$, but also a second matrix moment, namely
$\E[\Xb^\dagger\Xb^{\dagger\top}]$.  Even though $\X$ may not always
be full rank, $\Xb$ is full rank almost
surely (a.s.), so we can write
$\Xb^\dagger\Xb^{\dagger\top}=(\Xb^\top\Xb)^{-1}$.
\begin{theorem}\label{t:square-inverse}
Let $\X\sim\dxk$ and
  $\Xb\sim\vskx$ for any $d$-variate $\dx$. If
  $\rank(\X)\!=\!d$ a.s., then 
  \begin{align*}
   \E\big[\Xb^\dagger\Xb^{\dagger\top}\big]= \E\big[(\Xb^\top\Xb)^{-1}\big]\ \overset{(*)}{=}\
    \frac{k}{k-d+1}\cdot\big(\E[\X^\top\X]\big)^{-1}.
  \end{align*}
If $\rank(\X)<d$ with some probability then $(*)$ becomes a positive semi-definite
  inequality $\preceq$.
\end{theorem}
\begin{proof}
For a full rank $d\times d$ matrix $\A$ we have $\A^{-1}=\A^\dagger$
and $\adj(\A)=\det(\A)\A^{-1}$. When $\A$ is not full rank but psd, then
$\det(\A)\A^\dagger=\zero\preceq\adj(\A)$. Thus Lemma
\ref{l:determinant} implies that
\begin{align*}
  \E\big[(\Xb^\top\Xb)^{-1}\big]
  &= \frac{\E\big[\det(\X^\top\X)(\X^\top\X)^\dagger\big]}
{\E\big[\det(\X^\top\X)\big]}\\
  &\overset{(*)}{\preceq}\frac{\E\big[\adj(\X^\top\X)\big]}
    {\E\big[\det(\X^\top\X)\big]}\\
  \text{(Lemma
  \ref{l:determinant})}\quad
  &=\frac{(k^{\underline{d-1}}/k^{d-1})\cdot
                              \adj\!\big(\E[\X^\top\X]\big)}{(\ktd/k^d)
                              \cdot\det\!\big(\E[\X^\top\X]\big)}\\ 
&    =\frac{k}{k-d+1}\cdot\big(\E[\X^\top\X]\big)^{-1}, 
\end{align*}
where $(*)$ becomes an equality if $\X^\top\X$ is full rank with
probability 1.
\end{proof}

\subsection{Unbiased estimator for random design regression}
\label{s:unbiased}
In fixed design linear regression, given a fixed $k\times d$ matrix $\X$ and
a $k$-dimensional response vector $\y$, the least squares estimator
$\X^\dagger\y=\argmin_\w\|\X\w-\y\|^2$ is a canonical solution. When
the response vector is random, then the least squares solution satisfies
$\E[\X^\dagger\y]=\X^\dagger\E[\y]=\argmin_\w\E_{\y}\big[\|\X\w-\y\|^2\big]$, i.e., it is an
unbiased estimator of the minimizer of the expected square loss. In random
design regression, where each row-response pair is drawn independently
as $(\x^\top,y)\sim D$ from some $(d,1)$-variate population
distribution $D$, the matrix $\X\sim\dxk$ also becomes random. In this
context, the minimizer of the expected square loss is defined as
$\argmin_\w\E\big[(\x^\top\w-y)^2\big]=\sigd^{-1}\E[\x\,y]$. Note that
our assumption that $\rank(\sigd)=d$ comes without loss of generality
because the redundant components of vector $\x$ can be removed,
reducing dimension $d$ to match the rank of $\sigd$. The least
squares solution $\X^\dagger\y$ may no longer be an unbiased estimator of
the optimum under the random design model (in most cases it is not). We show that volume-rescaled
sampling provides a natural way of correcting the distribution $\dxk$
so that the least squares estimator is always unbiased.
\begin{theorem}\label{t:unbiased}
  Let $(\x^\top,y)\sim D$ be $(d,1)$-variate. Then for $\Xb\sim\vskx$ and
  $\yb_{i}\sim D_{{\cal Y}|\x=\xbb_{i}}$,
  \begin{align*}
    \E\big[\Xb^\dagger\ybb\big] 
    = \argmin_\w\E\big[(\x^\top\w-y)^2\big] = \w^*.
  \end{align*}
\end{theorem}
  \begin{proof}
Let $(\X,\y)\sim\dxy^k$. We first prove the
theorem for $k=d$. In this case, Cramer's rule implies that since $\X$
is a $d\times d$ matrix, we have
\begin{align*}
\det(\X^\top\X)\X^\dagger\y=\det(\X)\adj(\X)\,\y = \det(\X)\cdot
\begin{bmatrix}
\det(\X\!\overset{1}{\leftarrow}\!\y)\\
\vdots\\
 \det(\X\!\overset{d}{\leftarrow}\!\y)
\end{bmatrix},
\end{align*}
    where $\X\!\overset{i}{\leftarrow}\!\y$ is matrix $\X$
    with column $i$ replaced by $\y$. It follows that:
\begin{align*}
  \E\big[(\Xb^\dagger\ybb)_i\big]
  &=\frac{\E\big[\det(\X^\top\X)(\X^\dagger\y)_i\big]}{d!\det(\sigd)}\\
  &=\frac{\E\big[\det(\X)\det(\X\!\overset{i}{\leftarrow}\!\y)\big]}{d!\det(\sigd)}\\
\text{(Lemma~\ref{l:determinant})}\quad
&=\frac{\det\!\big(\,\E_D[\x\,(\x\!\overset{i}{\leftarrow}\!y)^\top]\,\big)}
           {\det(\sigd)}\\
    &= \frac{\det\!\big(\,\sigd \!\!\overset{i}{\leftarrow}\!\E[\x\,y]\,\big)}
            {\det(\sigd)}\\
&=\sigd^{-1}\,\E[\x\,y]=\argmin_\w\E\big[(\x^\top\w-y)^2\big].
\end{align*}
where we applied Lemma~\ref{l:determinant} to the pair of
    $d\times d$ matrices $\A=\X$ and
$\B=\X \overset{i}{\leftarrow} \y$.
The case of $k>d$ follows by induction based on the following lemma
shown by
\cite{unbiased-estimates-journal}:
\begin{lemma}\label{l:pseudoinverse} 
 For any matrix $\X\in\R^{k\times d}$, where $k> d$, denoting
 $\I_{-i}=\I-\e_i\e_i^\top$, we have
\begin{align*}
  \det(\X^\top\X)\,\X^\dagger =
  \frac1{k-d}\sum_{i=1}^k\det(\X^\top\I_{-i}\X)\,
  (\I_{-i}\X)^\dagger.
\end{align*}
\end{lemma}
Suppose that the induction hypothesis holds for $\Xt\sim\vskxm$ and $\yt_i\sim
D_{{\cal Y}|\x=\xbt_i}$. Then,
\begin{align*}
  \E\big[\Xb^\dagger\ybb\big] 
  &=\frac{\E\big[\det(\X^\top\X)
    \,\X^\dagger\y\big]}{\ktd\det(\sigd)}\\
&\!\overset{(a)}{=} \frac{\E
  \Big[\frac1{k-d}\sum_{i=1}^k
    \det(\X^\top\I_{-i}\X)\,(\I_{-i}\X)^\dagger\y\Big]}{\ktd\det(\sigd)}\\
  &= \frac1{k-d}\frac{\sum_{i=1}^k \E\big[\det(\X^\top\I_{-i}\X)
 (\I_{-i}\X)^\dagger\y\big]}{\ktd\det(\sigd)}\\
  &\!\overset{(b)}{=}\frac{k}{k-d}\,
    \frac{(k\!-\!1)^{\underline{d}}}{\ktd}\,\E\big[\Xt^\dagger\ybt\big]\ =\
\sigd^{-1}\E[\x\,y], 
\end{align*}
where $(a)$
follows from Lemma \ref{l:pseudoinverse}, while $(b)$ follows because
the rows of $\X\sim\dxk$ are exchangeable, so removing the $i$th row
is the same as removing the last row.
\end{proof}
The expected value of random matrix $\Xb^\dagger$ (Theorem \ref{t:pseudoinverse}) 
now follows by setting $y=1$:
\begin{proofof}{Theorem}{\ref{t:pseudoinverse}}
The columns of $\Xb^\dagger$, equal
$(\Xb^\top\Xb)^{-1}\xbb_i$, are exchangeable, so
\begin{align*}
  \E\big[(\Xb^\top\Xb)^{-1}\xbb_i\big] = \frac1k\cdot\E\big[\X^\dagger\one_k\big]\overset{(*)}{=}\frac1k\cdot\big(\E[\x\x^\top]\big)^{-1}\E[\x]=\big(\E[\X^\top\X]\big)^{-1}\E[\x],
\end{align*}
where $(*)$ is Theorem \ref{t:unbiased} with $y=1$.
The desired formula is the matrix form
of the above.
\end{proofof}
We now briefly discuss the implications of our method
in the case when the response variable is linear plus some
well-behaved noise.
More precisely, when the response values are modeled as $y_i=\x_i^\top\w^*+\xi_i$,
where $\E[\xi_i]=0$, $\Var[\xi_i]=\sigma^2$ and $\w^*\in\R^d$, then the covariance matrix of
the least squares estimator in fixed design regression is given by
$\Var[\X^\dagger\y\,|\,\X]=\sigma^2(\X^\top\X)^{-1}$ (here $\X$ is fixed). The covariance matrix
of the volume-rescaled sampling estimator in random design regression takes a similar form.
\begin{theorem}
  Let $(\x^\top,y)\sim \dx$ be $(d,1)$-variate. Suppose that
  $\E[y\,|\,\x]=\x^\top\w^*$ for some $\w^*\in\R^d$ and
  $\Var[y-\x^\top\w^*\ |\ \x]=\sigma^2$ almost surely. Then
  for $\Xb\sim\vskx$ and $\yb_{i}\sim D_{{\cal Y}|\x=\xbb_{i}}$,  
  \begin{align*}
    \Var\big[\Xb^\dagger\ybb\big]\overset{(*)}{=}
    \frac{k}{k-d+1}\cdot\sigma^2\big(\E[\X^\top\X]\big)^{-1},
    \quad\text{where }\X\sim\dxk,
  \end{align*}
  as long as $\rank(\X)=d$ almost surely, otherwise $(*)$ is replaced by inequality $\preceq$.
\end{theorem}
\begin{proof}
  Since $\E\big[\Xb^\dagger\ybb\big]=\E\big[\Xb^\dagger\,\E[\ybb\,|\,\Xb]\big]=\w^*$, denoting
  $\xib=\ybb-\Xb\w^*$, we have
  \begin{align*}
    \Var\big[\Xb^\dagger\ybb\big]
    &= \E\big[\Xb^\dagger(\Xb\w^*\!+\xib)
(\Xb\w^*\!+\xib)^\top\Xb^{\dagger\top}\big] - \w^*\w^{*\top} \\
    &=\E\big[\,\Xb^\dagger\,\E[\xib\xib^\top|\Xb]\,\Xb^{\dagger\top}\,\big]
      +
      \E\big[\Xb^\dagger\Xb\w^*\w^{*\top}(\Xb^\dagger\Xb)^\top\big]
      -\w^*\w^{*\top}\\
    &=\sigma^2\cdot\E\big[\Xb^\dagger\Xb^{\dagger\top}\big]\\
    &\overset{(*)}{=} \sigma^2\cdot
      \frac{k}{k-d+1}\big(\E[\X^\top\X]\big)^{-1}.
  \end{align*}
  Here, $(*)$ uses Theorem \ref{t:square-inverse}. It is
  replaced by $\preceq$ when $\rank(\X)<d$ with positive probability.
\end{proof}

\section{Loss bound for an unbiased estimator}
\label{s:loss}

For any distribution $\dxy$ defining a regression problem
$(\x^\top,y)\sim \dxy$, the quality of a vector $\w\in\R^d$ is measured by the expected square loss over $\dxy$:
\begin{align*}
  L_{\dxy}(\w) = \E\big[(\x^\top\w-y)^2\big].
\end{align*}
How many samples do we need to use to produce an \textit{unbiased} estimator $\wbh$ such
that the expected loss of $\wbh$ is no more than
$1+\epsilon$ times the optimum loss for the problem?
Concretely, given the input distribution $\dx$ and $\epsilon>0$, our goal is to find the
smallest $k$ for which there is a
$k\times d$-variate distribution $V_{\dx}^k$ and an estimator
$\wbh(\ybb|\Xb)$ such that 
$$\E\big[\wbh(\ybb|\Xb)\big]=\w^*,\quad\text{and}\quad
\E\big[L_{\dxy}\big(\wbh(\ybb|\Xb)\big)\big]\leq (1+\epsilon)L(\w^*),$$
where $\w^*=\argmin_\w L_{\dxy}(\w)$, $\Xb\sim V_{\dx}^k$ and $\yb_i\sim \dxy_{{\cal Y}|\x=\xbb_i}$.
Theorem \ref{t:unbiased} suggests that a natural candidate
for the sampling distribution $V_{\dx}^k$ of the $k$ points 
is volume-rescaled sampling $\vskx$ paired with the
estimator $\Xb^\dagger\ybb$.  Surprisingly we will show that this estimator
can have very large loss. Since the estimator does not depend on
the ordering of the rows of $\Xb$, it follows from Theorem
\ref{l:composition} that it can be equivalently constructed from a
volume-rescaled sample of size $d$ and an i.i.d.~sample of size $k-d$
from $\dx$. We denote such a sample as $\vsdx\cdot\dx^{k-d}$. Even
though this estimator is unbiased, most of the samples are coming from
the input distribution $\dx$, so if this distribution is particularly
ill-conditioned then we may not draw a point with high leverage until
a large number of samples were drawn. In the next
section, we present Theorem
\ref{t:lower-vskx} which implies the following lower
bound: \textit{For any $k\geq d$, there is a $(d,1)$-variate
  distribution  $D$ such that if $\Xb\sim\vskx$, then 
  $L_{\dxy}(\Xb^\dagger\ybb)\geq 2\cdot L_{\dxy}(\w^*)$ with probability at 
  least $0.25$.}

The standard solution for avoiding the case when the
examples have drastically different leverage scores is 
to replace the input distribution with the leverage score
distribution $\lev$. If the $k$ points are sampled i.i.d.\ from 
$\lev^k$ then it is known how to construct a \textit{biased} estimator 
which satisfies the $1+\epsilon$ loss bound for size $k=O(d\log d+d/\epsilon)$.  
In the below result we use a sampling distribution
consisting of a size $d$ volume-rescaled sample and a leverage score sample of size $k-d$,
i.e., the $k$ points are drawn from 
$\vsdx\!\cdot \lev^{k-d}$ to achieve both unbiasedness
and the loss bound with sample size $k=O(d\log d+d/\epsilon)$. 
The proof is broken down into two parts. The first part shows that the
loss bound holds when conditioned on a high probability event which
indicates when the leverage score sample is sufficiently well conditioned. This
part follows similarly to the standard analysis of leverage score sampling,
except we must additionally account for the negative dependence 
between the samples drawn by volume-rescaled sampling. The second part
of the proof analyzes the expected loss when the high probability
event fails. Here, standard analysis fails, and to address this, we
use a novel decomposition of the loss, relying on an expectation
inequality for volume-rescaled sampling (Lemma \ref{l:decorrelation}),
which is potentially of independent interest. In what follows, we use
$l_\x=\x^\top\sigd^{-1}\x$ to denote the leverage score of point $\x$.
\begin{theorem}\label{t:loss}
   Let $\dx$ be a $d$-variate distribution. For any $\epsilon >0$, there
   is  $k=O(d\log d+d/\epsilon)$ such that for any $D_{{\cal Y}|\x}$, if we sample $\Xb\sim \vsdx\!\cdot \lev^{k-d}$
   and $\yb_i\sim D_{{\cal Y}|\x=\xbb_i}$ then the estimator  $\wbh =
   \argmin_\w \sum_{i=1}^k \frac1{l_{\xbb_i}}(\xbb_i^\top\w -
   \yb_i)^2$ satisfies: 
   \begin{align*}
\quad \E[\wbh] &= \argmin_\w L_{\dxy}(\w)\quad\text{and}\quad \E\big[L_{\dxy}(\wbh)\big] \leq
     (1+\epsilon)\cdot \min_\w L_{\dxy}(\w).
     \end{align*}
   \end{theorem}
   
\begin{proof}
Let $\xbh^\top\sim \lev$ and $\yh\sim D_{{\cal Y}|\x=\xbh}$ jointly define
distribution $(\xbh^\top\!,\yh)\,\sim\,\dxyh$ and
\begin{align*}
  (\xbt^\top\!,\,\yt) = \bigg(\frac1{\sqrt{l_{\xbh}}}\,\xbh^\top\!,\ 
  \frac1{\sqrt{l_{\xbh}}}\,\yh\bigg)\ \sim\ \dxyt.
  \end{align*}
By Remark \ref{r:lev}, $\dxy$ and $\dxyt$ define the same loss function
up to a constant factor: 
\begin{align*}
  L_{\dxyt}(\w)=\E_{\lev}\!\Big[\frac 1{l_{\x}} \E_{\yh}\big[(\xbh^\top\w-\yh)^2\,|\,\xbh\big]\Big]
  =\E_{\dxy}\Big[\frac 1{l_{\x}}\big(\x^\top\w-y)^2\cdot
  l_{\x}\Big]\,/d = L_{\dxy}(\w)\,/d.
\end{align*}
Similarly, it follows that $\Sigmab_{\dxt}=\frac1d\sigd$. The key
property of distribution $\dxt$ is that it has uniform leverage
scores, implying that $\mathrm{Lev}_{\dxt}\!=\dxt$:
\begin{align}
\xbt^\top\Sigmab_{\dxt}^{-1}\xbt =\frac1{l_{\xbh}}\xbh^\top\Sigmab_{\dxt}^{-1}
  \xbh=\frac d{l_{\xbh}}\xbh^\top\sigd^{-1}\xbh=d.\label{eq:const-lev}
  \end{align}
    Let $\Xb$ and $\ybb$ be distributed as in the theorem. Note that we
can write the estimator $\wbh$ as follows:
\begin{align*}
\wbh=(\P_{\!\!\Xb}\Xb)^\dagger\P_{\!\!\Xb}\ybb,\quad\text{where}\quad
     \P_{\!\X}=\sum_{i=1}^k\frac1{\sqrt{l_{\x_i}}}\e_i\e_i^\top\!\in\R^{k\times k}.
\end{align*}
For any measurable function
$F(\P_{\!\!\Xb}\Xb, \P_{\!\!\Xb}\ybb)$, using Remarks \ref{r:expectations}
and \ref{r:lev}, as well as
$\det(\P_{\!\X})^2=\prod_{i=1}^k\frac1{l_{\x_i}}$ and $\det(\sigd)=\det(\Sigmab_{\dxt})d^d$, we obtain
\begin{align*}
\E\big[F(\P_{\!\!\Xb}\Xb, \P_{\!\!\Xb}\ybb)\big]
  &=\frac{\E_{\dx^k}\big[\E_{\y}[F(\P_{\!\X}\X,\P_{\!\X}\y)\,|\,\X] \cdot\det(\X_{[d]})^2\,\prod_{i=d+1}^kl_{\x_i}\big]}
    {d!\det(\sigd)\, d^{k-d}}&& (\X,\y)\sim \dxy^k
\\ &=\frac{\E_{\dx^k}\big[\E_{\y}[F(\P_{\!\X}\X, \P_{\!\X}\y)
     \,|\,\X]\cdot\det(\P_{\!\X_{[d]}}\X_{[d]})^2\prod_{i=1}^kl_{\x_i}\big]} 
     {d!\det(\Sigmab_{\dxt})d^d\, d^{k-d}}
\\ &=\frac{\E_{\dxh^k}\big[\E_{\ybh}[F(\P_{\!\!\Xh}\Xh,
     \P_{\!\!\Xh}\ybh)\,|\,\Xh]\cdot\det(\P_{\!\!\Xh_{[d]}}\!\Xh_{[d]})^2\big]}
     {d!\det(\Sigmab_{\dxt})}&&(\Xh,\ybh)\sim\dxyh^k,
\\
  &=\frac{\E_{\dxt^k}\big[\E_{\ybt}[F(\Xt,\ybt)\,|\,\Xt]\cdot\det(\Xt_{[d]})^2\big]}{d!\det(\Sigmab_{\dxt})}
&&(\Xt,\ybt)\sim\dxyt^k.
\end{align*}
This means that $\P_{\!\!\Xb}\Xb\sim\Vol_{\dxt}^d\!\!\cdot\dxt^{k-d}$ and
  $\P_{\!\xbb_i}\yb_i\sim \dxyt_{{\cal Y}|\x=\P_{\!\xbb_i}\!\xbb_i}$. So, since the losses
  $L_\dxy$ and $L_{\dxyt}$ are the same up to a constant factor and the estimator
  $\wbh=(\P_{\!\!\Xb}\Xb)^\dagger\P_{\!\!\Xb}\ybb$ is distributed identically to the
  corresponding estimator for $\dxyt$, proving the result for $\dxyt$ immediately implies the same for
  $\dxy$. Thus without loss of generality we can assume
  from now on that distribution $\dxy$ is the same as $\dxyt$, i.e. 
  we assume that $l_\x\!=d$ a.s.~for $\x\sim\dx$. This
 implies that $\lev\!=\dx$ and $\wbh=\Xb^\dagger\ybb$. 
    Also by Theorem \ref{l:composition}, matrix $\Xb\sim\vsdx\!\cdot\dx^{k-d}$
    after randomly reordering the rows becomes distributed as $\vskx$. 
    Thus by Theorem \ref{t:unbiased}, $\E[\Xb^\dagger\ybb]=\w^*$, 
    where $\w^*=\argmin_\w L_\dxy(\w)$,  showing the unbiasedness
    property of $\wbh$. 

    We are now ready to prove the loss bound.
      Note that $\E[(\x^\top\w^*\!-y)\,\x] =
\E[\x\x^\top]\w^*\!-\E[\x\,y]=\zero$, because
$\w^*=\sigd^{-1}\E[\x\,y]$. We use this to
perform a standard decomposition of the square loss:
\begin{align}
  L_{\dxy}(\w)
  &=
  \E_{\dxy}\big[(\x^\top\w-y)^2\big]\nonumber\\[-4mm]
&=\E\big[\big(\x^\top(\w-\w^*)\big)^2\big]
+
\overbrace{\E\big[\x^\top(\x^\top\w^*-y)\big]}^{\zero}(\w-\w^*)
+\E\big[(\x^\top\w^*-y)^2\big]\nonumber\\
  &=\E\big[\big(\x^\top(\w-\w^*)\big)^2\big] + L_{\dxy}(\w^*)\nonumber\\
  &=(\w-\w^*)^\top\E[\x\x^\top](\w-\w^*) + L_{\dxy}(\w^*) =
    \big\|\sigd^{\sfrac12}(\w-\w^*)\big\|^2 + L_{\dxy}(\w^*).\label{eq:loss-decomposition}
\end{align}
Substituting $\wbh=\Xb^\dagger\ybb=(\Xb^\top\Xb)^{-1}\Xb^\top\ybb$ for
$\w$, we additionally obtain:
\begin{align*}
  \big\|\sigd^{\sfrac12}(\wbh-\w^*)\big\|^2
  &=\big\|\sigd^{\sfrac12}(\Xb^\top\Xb)^{-1}\Xb^\top(\ybb-\Xb\w^*)\big\|^2\\
  &=\big\|(\sigd^{-\sfrac12}\Xb^\top\Xb \sigd^{-\sfrac12})^{-1}
    \sigd^{-\sfrac12}\Xb^\top(\ybb-\Xb
    \sigd^{-\sfrac12}\E[\sigd^{-\sfrac12}\x y])\big\|^2.
\end{align*}
Note that, without loss of generality, we can replace the distribution
$\x^\top\sim\dx$ by the distribution of $\x^\top\sigd^{-\sfrac12}$,
so from now on we will let $\sigd=\I$, in which case it suffices to bound $\E[\|\wbh-\w^*\|^2]=\E[\|(\Xb^\top\Xb)^{-1}\Xb^\top(\ybb-\Xb\w^*)\|^2]$.
A key step in the analysis is to ensure that the inverse
$(\Xb^\top\Xb)^{-1}$ is bounded. We can ensure that this is true with
high probability by relying on standard matrix Chernoff bounds, such
as the one stated below, essentially due to \cite{ahlswede2002strong}.
The particular version we use is adapted from \cite{chen2017condition}.
\begin{lemma}\label{l:matrix-tail}
There is a $C>0$, such that for any $\dx$ satisfying
$\x^\top\sigd^{-1}\x \leq K$ for all $\x\in\mathrm{supp}(\dx)$, if $\X\sim\dx^m$ and $m\geq C
K\epsilon^{-2}\log d/\delta$, then
\begin{align*}
    (1-\epsilon)\sigd\preceq\frac1m\X^\top\X\preceq
    (1+\epsilon)\sigd\quad\text{with probability }\geq 1-\delta.
\end{align*}
\end{lemma}

  Applying Lemma \ref{l:matrix-tail} for $\dx$ with $K=d$, $m=k-\lfloor k/2\rfloor$
and $\epsilon=1/2$ we obtain that if $k\geq d+4Cd\log d/\delta$
then the following event holds with probability $1-\delta$ with
respect to $\Xb\sim\vsdx\cdot\dx^{k-d}$ (where, recall that we let $\sigd=\I$):
\begin{align}
  \Ec:\qquad\Xb_{[s]^c}^\top\Xb_{[s]^c}\succeq\frac k4\cdot\I,\quad
  \text{ where }s=\lfloor k/2 \rfloor.\label{eq:embedding}
\end{align}
We next decompose the expectation $\E[\|\wbh-\w^*\|^2]$ into two
components, depending on whether event $\Ec$ occurs:
\begin{align}
  \E[\|\wbh-\w^*\|^2] = \Pr(\Ec)\cdot\E[\|\wbh-\w^*\|^2\mid \Ec] +
  \Pr(\neg\Ec)\cdot\E[\|\wbh-\w^*\|^2\mid \neg\Ec].
  \label{eq:decomposition}
\end{align}
The intuition here is that when $\Ec$ succeeds then this ensures a
strong control over the inverse $(\Xb^\top\Xb)^{-1}$ through
matrix concentration thanks to the i.i.d. sampled part of the matrix,
i.e., $\Xb_{[s]^c}\sim\dx^{k-s}$; whereas when $\Ec$ fails, then we can
still control the inverse by relying on the volume-rescaled sample
$\Xb_{[s]}\sim \vsdx\cdot \dx^{s-d}$. Here, thanks to the exponentially small failure
probability, $\Pr(\neg\Ec)$, we can rely on looser bounds for the
expectation.

\paragraph{Part 1: Event $\Ec$ succeeds}  We start by bounding the first
term in \eqref{eq:decomposition}, using a standard error decomposition
\citep[see Lemma 1 of][]{drineas2011faster}:
\begin{align*}
  \Pr(\Ec)\cdot\E[\|\wbh-\w^*\|^2\mid \Ec]
  &\leq
    \Pr(\Ec)\E\big[\|(\Xb^\top\Xb)^{-1}\|^2
    \|\Xb^\top(\ybb-\Xb\w^*)\|^2\mid\Ec\big]
  \\
  &\leq \frac{4^2}{k^2}
    \Pr(\Ec)\E\big[\|\Xb^\top(\ybb-\Xb\w^*)\|^2\mid\Ec\big]
  \\
  &\leq \frac{4^2}{k^2}\E\big[\|\Xb^\top(\ybb-\Xb\w^*)\|^2\big],
\end{align*}
where we used that $\|(\Xb^\top\Xb)^{-1}\|\leq
\|(\Xb_{[s]^c}^\top\Xb_{[s]^c})^{-1}\|\leq 4/k$, when conditioned on
$\Ec$.

We next bound the expectation $\E[\|\Xb^\top(\ybb-\Xb\w^*)\|^2]$. Unlike with
i.i.d.~leverage score sampling, this requires controlling the pairwise dependence
between indices because of the jointness of volume-rescaled
sampling. Denoting $\rbb=\ybb-\Xb\w^*$, and observing that vectors
$\Xb_{[d]}^\top\rbb_{[d]}$, $\xbb_{d+1}\rb_{d+1},\dots,
\xbb_{k}\rb_{k}$ are independent and mean zero, we have 
\begin{align}
  \E\big[\big\|\Xb^\top\rbb\big\|^2\big]
       &=
  \E\big[\big\|\Xb_{[d]}^\top\rbb_{[d]}\big\|^2\big] 
+\sum_{i\in[d]^c}\E\big[\|\xbb_i\rb_i\|^2\big]\nonumber
  \\ &=
       \sum_{i,j\in[d]}\E\big[\rb_i\rb_j\xbb_i^\top\xbb_j\big]
       + (k-d)\,\E\big[d\,(y-\x^\top\w^*)^2\big]\nonumber
  \\ &= d(d\!-\!1)\,\E\big[\rb_1\rb_2\xbb_1^\top\xbb_2\big]
+       d^2L_{\dxy}(\w^*)+(k-d)d\,L_{\dxy}(\w^*).\label{eq:marginals0}
\end{align}
The only difference in using volume-rescaled sampling rather than just
$\dx^k$ is the presence of the first term in \eqref{eq:marginals0},
which would be zero if the rows were fully independent. We will show
that due to the negative dependence of $\vsdx$ this term is in
fact non-positive. We rely on the
following lemma which describes the marginal distribution of subsets
of rows in volume-rescaled sampling of size $d$ by relying on known
properties of determinantal point processes \citep[see Proposition 19 in][]{dpp-independence}.
\begin{lemma}\label{l:joint}
The marginal distribution of $t$ rows of
$\Xb\sim\vsdx$ indexed by $T\subseteq[d]$ is
\begin{align*}
\Pr\big(\Xb_T\!\in\! A\big)=
  \E_{\dx^t}\big[\one_{[\X_T\in A]}\cdot \det\!\big(\X_T\sigd^{-1}\X_T^\top\big)\big]\,/d^{\underline{t}},
\end{align*}
where $A\subseteq\R^{t\times d}$ is measurable w.r.t.~$\dx^t$.
\end{lemma}
We apply Lemma \ref{l:joint} to the set
$T=\{1,2\}$ and compute the determinant of a $2\times 2$
matrix:
\begin{align*}
  \det(\X_T\sigd^{-1}\X_T^\top)
&=l_{\x_1}l_{\x_2}
- (\x_1^\top\sigd^{-1}\x_2)^2,
\end{align*}
Recall that we assumed $l_\x=d$ for $\x\sim \dx$, and $\sigd=\I$. We next show that
the first term in \eqref{eq:marginals0} is non-positive, so the pairwise
  dependence between the rows in volume-rescaled sampling can only 
  improve the bound. Denoting
  $r_i=y_i-\x_i^\top\w^*$, we have
  \begin{align*}
    d(d\!-\!1)\,\E\big[\rb_1\rb_2\xbb_1^\top\xbb_2\big]
    &=
      d(d\!-\!1)\,\E_{\dxy^2}\big[r_1r_2\x_1^\top\x_2\,\det(\X_T\Sigmab_{\dx}^{-1}\X_T^\top)\big]/d^{\underline{2}}
      \\
     &=\E_{\dxy^2}\big[r_1r_2\x_1^\top\x_2
      \big(d^2 -(\x_1^\top\x_2)^2\big)\big]
\\ &= d^2\big\|\underbrace{\E_{\dxy}[\x\, (y-\x^\top\w^*)]}_{\zero}\big\|^2-
     \underbrace{\E_{\dxy^2}\big[r_1r_2(\x_1^\top\x_2)^3\big]}_{E}.
  \end{align*}
$E$ can be written as a sum
  $\sum_c\E_{\dxy^2}[f_c(\x_1,y_1)f_c(\x_2,y_2)]=\sum_c(\E_{\dxy}[f_c(\x_1,y_1)])^2\geq 0$, where $f_c(\cdot)$ is
  some expression of its arguments, because $(\x_1,y_1)$ and
  $(\x_2,y_2)$ are independent and identically distributed.

Altogether, we obtained that
    $\E\big[\|\Xb^\top\rbb\|^2\big]\leq kd\,L_{\dxy}(\w^*)$, which in
    turn implies that
    \begin{align*}
        \Pr(\Ec)\cdot\E[\|\wbh-\w^*\|^2\mid \Ec]\leq \frac{4^2d}{k}\, L_{\dxy}(\w^*).
    \end{align*}

    \paragraph{Part 2: Event $\Ec$ fails} Let us again use
    the notation of $\rbb=\ybb-\Xb\w^*$. To bound the  second term in
    \eqref{eq:decomposition}, we use a somewhat different
    decomposition of $\|\wbh-\w^*\|$ than we did in Part 1:
\begin{align*}
  \|\wbh-\w^*\|^2
  &= \|\Xb^\dagger\rbb\|^2 \leq
  \|\Xb^\dagger\|^2\cdot\|\rbb\|^2
 =\|(\Xb^\top\Xb)^{-1}\|\cdot\big(\|\rbb_{[s]}\|^2 +
    \|\rbb_{[s]^c}\|^2\big)
  \\
  &\leq \tr\big((\Xb_{[s]}^\top\Xb_{[s]})^{-1}\big)\cdot\big(\|\rbb_{[s]}\|^2 +
    \|\rbb_{[s]^c}\|^2\big).
\end{align*}
So, taking expectation, and noting that $\Xb_{[s]}$ and $\rbb_{[s]}$
are independent of $\Ec$, we have:
\begin{align*}
  \E\big[\|\wbh-\w^*\|^2\mid\neg\Ec\big]\leq
  \E\big[\tr((\Xb_{[s]}^\top\Xb_{[s]})^{-1})\|\rbb_{[s]}\|^2\big]
  +\E\big[\tr((\Xb_{[s]}^\top\Xb_{[s]})^{-1})\big]\, \E\big[\|\rbb_{[s]^c}\|^2\mid\neg\Ec\big].
\end{align*}
Thus, we are able to restrict the conditioning on $\neg\Ec$ to only
the term $\E\big[\|\rbb_{[s]^c}\|^2\mid\neg\Ec\big]$, which allows us
to analyze the remaining terms as if they were distributed according
to volume-rescaled sampling, without the distribution being distorted
by the conditioning. In particular, using Theorem \ref{t:square-inverse} we obtain that:
\begin{align*}
  \E\big[\tr((\Xb_{[s]}^\top\Xb_{[s]})^{-1})\big] \leq
  \frac{d}{s-d+1}\leq \frac{3d}{k}.
\end{align*}
Next, with a slight abuse of notation, assume that the rows of
$\Xb_{[s]}$ are permuted (i.e., that $\Xb_{[s]}\sim
{\mathrm{VS}_{\!\dx}^s}$) so that they are identically
distributed. Then, we have:
\begin{align*}
  \E\big[\tr((\Xb_{[s]}^\top\Xb_{[s]})^{-1})\|\rbb_{[s]}\|^2\big]
  &=\sum_{i=1}^s\E\big[\rb_i^2\,\tr((\Xb_{[s]}^\top\Xb_{[s]})^{-1})\big]
=    s\cdot
    \E\big[\rb_s^2\,\tr((\Xb_{[s]}^\top\Xb_{[s]})^{-1})\big].
\end{align*}
To apply Theorem
\ref{t:square-inverse} again, we must disentangle the trace from
$r_s^2$, which is addressed in the following lemma proven at the end
of the section.
\begin{lemma}\label{l:decorrelation}
  If $\Xb\sim\vskx$, where $\sigd=\I$, then for any
  $f:\R^d\rightarrow \R_{\geq 0}$ and $i\in[k]$, 
  \begin{align*}
    \E\big[f(\xbb_i)\,\tr((\Xb^\top\Xb)^{-1})\big] \overset{(*)}\leq \frac dk\cdot \E_{\dx}\big[f(x)\big] +
    \frac{d-1}{k(k-d+1)}\cdot \E_{\dx}\big[\|\x\|^2f(\x)\big],
  \end{align*}
  as long as the right-hand side is well-defined, where $(*)$ becomes
  an equality if $\X\sim\dxk$ is a.s.~rank $d$. If we also assume
  that $\|\x\|^2=d$ a.s.~for $\x^\top\sim\dx$, then we get
  \[\E\big[f(\xbb_i)\,\tr((\Xb^\top\Xb)^{-1})\big]=\E[f(\xbb_i)]\E[\tr((\Xb^\top\Xb)^{-1})]
    = \frac d{k-d+1}\E_{\dx}[f(\x)].\]
\end{lemma}
Using the lemma with $f(\xbb_s)=\E[\rb_s^2\mid \xbb_s]$, we conclude that:
\begin{align*}
  \E\big[\tr((\Xb_{[s]}^\top\Xb_{[s]})^{-1})\|\rbb_{[s]}\|^2\big] \leq
  \frac{s d}{s-d+1}\,  L_{\dxy}(\w^*)\leq 2d \,  L_{\dxy}(\w^*).
\end{align*}
It remains to bound the final term,
$\E\big[\|\rbb_{[s]^c}\|^2\mid\neg\Ec\big]$. To that end, we define an
additional event $\Ec'$ as follows:
\begin{align*}
  \Ec':\qquad \Xb_{[s+1,k-1]}^\top\Xb_{[s+1,k-1]}
  \succeq \frac k4\cdot \I.
\end{align*}
Note that $\Ec'$ implies $\Ec$, and we can easily use
Lemma~\ref{l:matrix-tail} to bound its failure probability.
Also, observe that, since the marginal distribution of each vector $\xbb_i$
for $i\in[s]^c$ is the same, and the event $\Ec$ is invariant under
permutation of the indices of these vectors, the marginal
distributions of $\rb_i^2=(\yb_i-\xbb_i^\top\w^*)^2$ conditioned on $\neg\Ec$ are the
same for each $i\in[s]^c$, so:
\begin{align*}
  \E\big[\|\rbb_{[s]^c}\|^2\mid\neg\Ec\big]
  &= \sum_{i=s+1}^k\E[\rb_i^2\mid\neg\Ec]
\leq k\cdot\E[\rb_k^2\mid\neg\Ec]
=k\cdot \frac{\E[\rb_k^2\cdot\one_{\neg\Ec}]}{\Pr(\neg\Ec)}
  \\
  &\leq k\cdot \frac{\E[\rb_k^2\cdot\one_{\neg\Ec'}]}{\Pr(\neg\Ec)}
  =k\cdot \frac{\E[\rb_k^2]\Pr(\neg\Ec')}{\Pr(\neg\Ec)} =
   k\, \frac{\Pr(\neg\Ec')}{\Pr(\neg\Ec)}\,L_{\dxy}(\w^*),
\end{align*}
where we used the fact that $\Ec'$ is independent of $\rb_k$. Putting
everything together, we conclude that:
\begin{align*}
  \Pr(\neg\Ec)\cdot\E\big[\|\wbh-\w^*\|^2\mid\neg\Ec\big]
  &\leq \Pr(\neg\Ec)\cdot\Big(2d\,L_{\dxy}(\w^*) + \frac{3d}{k}\cdot
    k\,\frac{\Pr(\neg\Ec')}{\Pr(\neg\Ec)}\, L_{\dxy}(\w^*) \Big)
  \\
  &\leq \Pr(\neg\Ec') 5d\, L_{\dxy}(\w^*).
\end{align*}
It remains to note that, setting $\delta = 1/k$ in Lemma
\ref{l:matrix-tail}, we can ensure that $\Pr(\neg\Ec')\leq 1/k$ for
$k\geq C' d\log(dk)$ with sufficiently large constant $C'$. This can
be easily converted to a condition of the form $k\geq C''d\log
d$. Under this condition, combining Part 1 and Part 2, we obtain the following bound: 
\begin{align*}
  \E[L_{\dxy}(\wbh)] - L_{\dxy}(\w^*) =\E\big[\|\wbh-\w^*\|^2\big]\leq \frac{9d}k\,L_{\dxy}(\w^*) +
  \frac{5d}k\,L_{\dxy}(\w^*) =\frac {14d}k\,L_{\dxy}(\w^*),
\end{align*}
which concludes the proof.
\end{proof}

Note that, using Markov's inequality, we can convert the expected loss
bound to a bound that holds with high probability. Namely, sample
size $O(d\log d + d/(\epsilon\delta))$ suffices to obtain that
$L_{\dxy}(\wbh)\leq (1+\epsilon)\, L_{\dxy}(\w^*)$ holds with
probability $1-\delta$.

The above result can also be achieved if we replace the exact leverage
score sampling distribution with its approximation. As discussed in
Section \ref{s:algs}, producing samples from such approximation can be
more practical in settings where exact leverage scores are too
expensive to compute.
\begin{lemma}\label{l:loss}
  Theorem \ref{t:loss} still holds if we replace $l_\x$ with any
  $\hat{l}_\x$ such that $\frac12 l_\x\leq \hat{l}_\x\leq \frac32
  l_\x$ for all $\x^\top\!\in\mathrm{supp}(\dx)$ and also replace
  $\lev$ with the following $d$-variate distribution:
  \begin{align*}
   \levhh(A) \defeq \frac{\E_{\dx}\big[\one_{[\x^\top\in A]}\,\hat{l}_\x\big]}{\E_{\dx}\big[\hat{l}_\x\big]}.
  \end{align*}
\end{lemma}
The proof presented in Appendix \ref{a:loss}, follows a similar
outline as for Theorem \ref{t:loss}, however it has some additional
steps because when $\levhh\neq \lev$ then the marginal
distribution of volume-rescaled sampling $\vsdx$ (which is still
$\lev$, see Theorem \ref{t:marginal}) is no longer $\levhh$.

\begin{proofof}{Lemma}{\ref{l:decorrelation}}
 Since the rows of $\Xb$ are exchangeable, without loss of generality
 assume that $i=1$. By definition of volume-rescaled sampling, we have:
  \begin{align*}
    \E\big[f(\xbb_1)\,\tr((\Xb^\top\Xb)^{-1})\big]
    \leq
    \frac{\E[f(\x_1)\,\tr(\adj(\X^\top\X))]}{\E[\det(\X^\top\X)]},\quad\text{for
    }\X\sim \dxk.
  \end{align*}
 We next derive a recursion for the numerator in the above
 expression. To that end, let $F(k)
 =\E[f(\x_1)\,\tr(\adj(\X^\top\X))]$. As a simple consequence of the
 Cauchy-Binet formula, we have
 $\adj(\X^\top\X) =
 \frac1{k-d+1}\sum_{i=1}^k\adj(\X_{-i}^\top\X_{-i})$ for any $k\geq d$, so:
 \begin{align*}
   F(k)
   &=
   \frac1{k-d+1}\sum_{i=1}^k\E\big[f(\x_1)\,\tr(\adj(\X_{-i}^\top\X_{-i}))\big]
   \\
   &=\E[f(\x_1)]\frac{\E[\tr(\adj(\X_{-1}^\top\X_{-1}))]}{k-d+1} +
     \frac{k-1}{k-d+1}\E\big[f(\x_1)\,\tr(\adj(\X_{-k}^\top\X_{-k}))\big]
   \\
   &\overset{(a)}{=}\E_{\dx}[f(\x)]\frac{\frac{(k-1)!}{(k-d)!}\tr(\adj(\sigd))}{k-d+1} +
     \frac{k-1}{k-d+1}F(k-1)
   \\
   &\overset{(b)}{=}\frac{(k-1)!}{(k-d+1)!}\,d\,\E_{\dx}[f(\x)]   +
     \frac{k-1}{k-d+1}F(k-1)
   \\
   &\overset{(c)}{=}\frac{(k-1)!}{(k-d)!}\,d\,\E_{\dx}[f(\x)] + {k-1\choose d-2}\,F(d-1),
 \end{align*}
 where in $(a)$ we used Lemma \ref{l:determinant}, $(b)$ follows because of the
 assumption that $\sigd=\I$, and in $(c)$ we unroll the recursion on
 $F(k)$ for as long as the Cauchy-Binet can be applied to the adjugate
 matrices. To compute $F(d-1)$, we use the definition of the adjugate
 matrix, together with the formula $\det(\A+\v\v^\top)=\det(\A)
 +\v^\top\adj(\A)\v$. Suppose that $\X\sim\dx^{d-1}$,  and let
 $j\in[d]$. Before we compute the desired expectation formula for the
 trace, we first derive the expectation formula for the $j$th diagonal
 element of the corresponding matrix:
 \begin{align*}
   \E\big[f(\x_1)\adj(\X^\top\X)_{jj}\big]
   &=\E\big[f(\x_1)\det((\X^{-j})^\top\X^{-j})\big]
   \\
   &=\E\big[f(\x_1)\det((\X_{-1}^{-j})^\top\X_{-1}^{-j}+\x_1^{-j}(\x_1^{-j})^\top)\big]
   \\
   &\overset{(a)}{=}\E\big[f(\x_1)(\x_1^{-j})^\top\E[\adj((\X_{-1}^{-j})^\top\X_{-1}^{-j})]\x_1^{-j}\big]
   \\
   &\overset{(b)}{=} (d-2)!\,\E\big[f(\x_1)(\x_1^{-j})^\top\!\adj(\I_{d-1}) \x_1^{-j}\big]
   \\
   &=(d-2)!\,\E_{\dx}\big[f(\x)\|\x^{-j}\|^2\big],
 \end{align*}
 where $\x^{-j}$ denotes vector $\x$ without the $j$th entry and
 $\X^{-j}$ denotes matrix $\X$ without the $j$th column, $(a)$
 follows because $\det((\X_{-1}^{-j})^\top\X_{-1}^{-j})=0$ 
 and $(b)$ comes from Lemma~\ref{l:determinant}. Finally, to compute
 the trace, we sum up over $j$:
 \begin{align*}
   F(d-1)
   &= \E\big[f(\x_1)\tr(\adj(\X^\top\X))\big]
=\sum_{j=1}^d  \E\big[f(\x_1)\adj(\X^\top\X)_{jj}\big]
   \\[-2mm]
   &=(d-2)!\sum_{j=1}^d \E_{\dx}\big[f(\x)\|\x^{-j}\|^2\big]
     =(d-2)!\sum_{j=1}^d\sum_{l\neq j}\E_{\dx}\big[f(\x)(\x^{j})^2\big]
   \\
   &=(d-2)!\cdot (d-1)\,\E_{\dx}\big[f(\x)\|\x\|^2\big] = (d-1)!\,\E_{\dx}\big[f(\x)\|\x\|^2\big].
 \end{align*}
 Finally, recalling from Lemma \ref{l:determinant} that
 $\E_{\dxk}[\det(\X^\top\X)] = \frac{k!}{(k-d)!}\det(\sigd)$, we
 obtain that for $\X\sim\dxk$:
 \begin{align*}
   \frac{\E[f(\x_1)\,\tr(\adj(\X^\top\X))]}{\E[\det(\X^\top\X)]}
   &=\frac
     {\frac{(k-1)!}{(k-d)!}}{\frac{k!}{(k-d)!}}\,d\,\E_{\dx}[f(x)]+
     \frac{{k-1\choose d-2}}{\frac{k!}{(k-d)!}} (d-1)!\,
     \E_{\dx}\big[f(\x)\|\x\|^2\big]
   \\
   &=\frac dk\,\E_{\dx}[f(x)] + \frac{d-1}{k(k-d+1)}\, \E_{\dx}\big[f(\x)\|\x\|^2\big],
 \end{align*}
 which completes the proof of the claim. Note that, analogously as in
 Theorem \ref{t:square-inverse}, under the assumption that
 $\rank(\X)=d$ almost surely, we can replace the inequality in the
 statement by an equality. If we additionally let $\|\x\|^2=d$ almost surely for
 $\x^\top\sim\dx$, which due to the assumption that $\sigd=\I$
 corresponds to the distribution $\dx$ having uniform leverage scores,
 then the result can be stated in a simpler way:
 \begin{align*}
   \E_{\vskx}\big[f(\xbb_i)\tr((\Xb^\top\Xb)^{-1})\big]
   &=\E_{\dx}[f(\x)]\Big(\frac dk + \frac{d(d-1)}{k(k-d+1)}\Big)
   \\
   &=\E_{\dx}[f(\x)]\,\frac d{k-d+1}
   \\
   &=\E_{\vskx}[f(\xbb_i)]\cdot \E_{\vskx}[\tr((\Xb^\top\Xb)^{-1})],
 \end{align*}
 which implies that random variables $f(\xbb_i)$ and
 $\tr((\Xb^\top\Xb)^{-1})$ are uncorrelated.
\end{proofof}

\section{Lower bounds} \label{s:lower}
In this section we present lower bounds demonstrating the
limitations of the least squares estimator under certain random
designs, starting with $\X\sim\dx^k$ which
samples $k$ points directly from the data distribution. The key
shortcoming of the least squares estimator $\X^\dagger\y$ in this
context is that it is usually biased. In particular, this means that the loss
of the mean of that estimator, $L_{\dxy}\big(\E[\X^\dagger\y]\big)$,
is larger than the minimum loss $L(\w^*)$, where $\w^*=\argmin_\w L_{\dxy}(\w)$. 
We next show that for some distributions $\dxy$ this bias can be quite significant.
 \begin{theorem}\label{t:bias}
   Let $(\x^\top,y)\sim\dxy$ be a $(d,1)$-variate distribution
   s.t.~$(\x^\top,y)= (Z\e_{J}^\top,Z^3)$ for $Z\sim \Nc(0,1)$ and 
     $J\sim\mathrm{Uniform}\big([d]\big)$ drawn independently.
 Then, for any $k\geq0$ and $(\X,\y)\sim\dxy^k$,
 \begin{align*}
   \E[\X^\dagger\y] &= (1-\delta) \cdot \w^*
\quad\quad\text{and}\quad\quad  L_{\dxy}\big(\E[\X^\dagger\y]\big)
   = \bigl(1 + \tfrac32\delta^2\bigr)
    \cdot L_\dxy (\w^*), \\[1mm]
\text{where}\quad  \delta &= \frac{2d}{k+1} \cdot
  \bigg( 1 - \frac{d}{k+2} + \frac{d-1}{k+2} \cdot \Big(1 - \frac1d \Big)^{k+1} \bigg) .
 \end{align*}
    \end{theorem}
\begin{proof}
  Since $\E[\x\x^\top] = (1/d) \I$ and $\E[y\x] = \E[Z^4\e_{J}] = (3/d,\dotsc,3/d)$, it follows that $\w^* = (3,\dotsc,3)$.
  For any $c \in \R$, the loss of $(1-c) \cdot \w^*$ is $L_{\dxy}((1-c) \cdot \w^*) = \E[ (Z^3 - 3(1-c)Z)^2 ] = 6 + 9c^2 = (1+3c^2/2) \cdot L_{\dxy}(\w^*)$.

  It remains to show that $\E[\X^\dagger\y] = (1-\delta) \cdot \w^*$, i.e., each entry of $\X^\dagger\y$ has expectation $3 \cdot (1-\delta)$.
  Let us write $\x_i = Z_i \e_{J_i}$ and $y_i = Z_i^3$ for $i=1,\dotsc,k$, where $(Z_i,J_i)$ for $i=1,\dotsc,k$ are independent copies of $(Z,J)$.
  Furthermore, let $S_j := \{ i \in [k] : J_i = j \}$ for $j=1,\dotsc,d$.
  Then $\X^\top\X$ is a diagonal matrix whose $(j,j)$-th entry is $\sum_{i \in S_j} Z_i^2$, and $\X^\top\y$ is a vector whose $j$-th entry is $\sum_{i \in S_j} Z_i^4$.
  Therefore, the $j$-th entry of $\X^\dagger\y$ is
  \[ (\X^\dagger\y)_j = \frac{\sum_{i \in S_j} Z_i^4}{\sum_{i \in S_j} Z_i^2} . \]
  Here, we use the convention $0/0 = 0$ to handle the possibility of $S_j = \emptyset$.

  We first condition on $S_j$, and then take expectation with respect to the $Z_i$'s.
  For notational convenience, assume $S_j = \{1,\dotsc,m\}$.
  Recall that the joint distribution of $(Z_1,\dotsc,Z_m)$ is the same as that of $L \cdot \u$, where $L^2$ is a $\chi^2$ random variable with $m$ degrees of freedom, $\u = (u_1,\dotsc,u_m)$ is uniformly distributed on the unit sphere in $\R^m$, and $L^2$ and $\u$ are independent.
  Then
  \[
    \E\left[ \frac{\sum_{i=1}^m Z_i^4}{\sum_{i=1}^m Z_i^2} \right]
    = \E\left[ \frac{L^4 \sum_{i=1}^m u_i^4}{L^2 \sum_{i=1}^m u_i^2} \right]
    \overset{(a)}{=} \E\left[ L^2 \sum_{i=1}^m u_i^4 \right]
    \overset{(b)}{=} m^2 \cdot \E[ u_1^4 ]
    \overset{(c)}{=} m^2 \cdot \frac{3}{m(m+2)}
    .
  \]
  Above, $(a)$ uses the fact that $\sum_{i=1}^m u_i^2 = 1$;
  $(b)$ uses the independence of $L^2$ and $\u$, symmetry, and the fact $\E[L^2]=m$;
  and $(c)$ follows from Proposition~\ref{prop:uniform-moments}.
  Therefore, returning to the original notation, we have
  \[ \E\left[ (\X^\dagger\y)_j \mid S_j \right] =  3 \cdot \left( 1 - \frac{2}{|S_j|+2} \right) . \]
  (Note that this is consistent with the case where $S_j = \emptyset$.)

  Now we take expectation with respect to $S_j$.
  Observe that $|S_j|$ is Bernoulli-distributed with $k$ trials and success probability $\Pr(J = j) = 1/d$.
  Therefore, using the probability generating function for $|S_j|$, which is given by $G(t) := (1-1/d+t/d)^k$, we have
  \[
    \E\left[ \frac{2}{|S_j|+2} \right]
    = 2 \int_0^1 t \cdot G(t) \dif t
    = 2 \cdot \frac{d(k-d+2) + (d-1)^2 (1-1/d)^k}{(k+1)(k+2)}
    = \delta
  \]
  \citep[see, e.g.,][]{chao1972negative}.
  So we conclude $\E[ (\X^\dagger \y)_j ] = 3 \cdot (1-\delta)$.
\end{proof}
In Section~\ref{s:unbiased} we showed that a random design based on
volume-rescaled sampling, $\Xb\sim\vskx$, makes the least squares
estimator unbiased for all distributions $\dxy$. Recall that by
Theorem \ref{l:composition} the same estimator can also be obtained from
$\Xb\sim\vsdx\cdot\dx^{k-d}$. Despite offering unbiasedness, this
random design does not guarantee
strong loss bounds. This forced us to combine
volume-rescaled sampling with leverage score sampling  in Section
\ref{s:loss}, obtaining distribution $\vsdx\cdot\lev^{k-d}$. 
The following lower bound
shows that the loss bound obtained for this random design (Theorem
\ref{t:loss}) cannot be achieved by vanilla volume-rescaled sampling $\vskx$. 
This general lower bound can also be
easily adapted to the previously studied variants of discrete volume sampling from
finite datasets \citep{avron-boutsidis13,unbiased-estimates-journal}.

\begin{theorem}\label{t:lower-vskx}
  Let $(\x^\top,y)\sim\dxy$ be a $(d,1)$-variate distribution for which:
  \begin{align*}
    (\x^\top,y) = \begin{cases}
      (\e_i^\top,1)&\text{ for each $i\in[d]$ with probability}\ \frac\delta d,\\
      (\gamma\e_i^\top,0)&\text{ for each $i\in[d]$ with probability}\ \frac{1-\delta}{d}.
    \end{cases}
  \end{align*}
For any $k\geq d$, there is $\gamma,\delta\in(0,1)$ such
  that if $\Xb\sim\vskx$ and
  $\yb_{i}\sim D_{{\cal Y}|\x=\xbb_{i}}$, then
  \begin{align*}
      \Pr\Big(L_{\dxy}\big(\Xb^\dagger\ybb\big)\geq 2\cdot\min_\w
    L_{\dxy}(\w)\Big)\geq 0.25.
   \end{align*}
 \end{theorem}
 Note that the above statement immediately implies a
 lower bound for the \emph{expected} loss of the  estimator $\Xb^\dagger\ybb$, namely,
 that  $\E\big[L_\dxy(\Xb^\dagger\ybb)\big]-L_\dxy(\w^*)\geq
0.25\cdot L_\dxy(\w^*)$.  This shows that the guarantee in
 Theorem~\ref{t:loss}  cannot be established  for vanilla volume-rescaled
 sampling with $\epsilon<0.25$.
\begin{proof}
First, we find $L_{\dxy}(\w^*)$. Simple calculations show that: 
\begin{align*}
\sigd&= \frac{\delta+\gamma^2(1-\delta)}{d}\, \I
\quad
\text{and}\quad \w^* =
            \frac{\delta}{\delta+\gamma^2(1-\delta)}\one_d,\quad \text{so}\\[2mm]
L_{\dxy}(\w^*) &= \delta\,(1-\e_1^\top\w^*)^2 + (1-\delta)\,(\gamma\e_1^\top\w^*)^2
  = \frac{\gamma^2\delta(1-\delta)}{\delta+\gamma^2(1-\delta)}.  
\end{align*}
Let $A_{\Xb}$ denote the event that there exists
$j\in[d]$ such that no vector 
$\xbb_i$ is equal to $\e_j$. If $A_{\Xb}$ holds then the $j$th component of
$\Xb^\dagger\ybb$ is $0$ so, setting $\gamma^2=
\frac\delta{2d(1-\delta)}$, 
\begin{align*}
 L_{\dxy}(\Xb^\dagger\ybb)\geq \frac\delta d =
  2\,\frac{\gamma^2\delta(1-\delta)}{\delta}\geq
  2\,\frac{\gamma^2\delta(1-\delta)}{\delta+\gamma^2(1-\delta)}=2\,
  L_{\dxy}(\w^*)\qquad\text{(conditioned on $A_{\Xb}$)}.
\end{align*}
It remains to lower bound the probability of $A_{\Xb}$. We use Theorem
\ref{l:composition} to decompose $\Xb$ into $\Xb_S\sim\vsdx$ and
$\Xb_{S^c}\sim\dx^{k-d}$. Setting $\delta=\frac d{4k}$, we obtain:
\begin{align*}
  \Pr(A_{\Xb})&\overset{(a)}{\geq} \Pr(A_{\Xb_S})\,\Big(1-\frac\delta d\Big)^{k-d}
\\ &\overset{(b)}{=}\bigg(1-\frac{\det(\I)}{d!\det(\sigd)}\cdot d! \Big(\frac\delta
                d\Big)^d\bigg)\Big(1-\frac\delta d\Big)^{k-d}
  \\ &=\bigg(1-\frac1{(1+\gamma^2\frac{1-\delta}{\delta})^d}\bigg)
       \Big(1-\frac\delta d\Big)^{k-d}
\\ & =\bigg(1-\frac1{(1+\frac1{2d})^d}\bigg)
     \Big(1-\frac\delta d\Big)^{k-d}
  \\ &\overset{(c)}{\geq} \bigg(1-\frac1{1+\frac1{2}}\bigg)
            \Big(1-\delta\,\frac{k-d}d\Big)\ \geq\ \frac13 \cdot
       \frac34\ =\ \frac14,
\end{align*}
where $(a)$ follows because if some unit vector $\e_j$ is missed by
$\Xb_S$ and it is not selected by any of the $k-d$ i.i.d.~samples then
$A_{\Xb}$ holds. In $(b)$, factor $d!(\frac\delta d)^d$ is the
probability of selecting some row-permutation of the identity matrix in $\dx^d$.
Finally, $(c)$ is Bernoulli's inequality applied twice.
\end{proof}

\section{Algorithms}
\label{s:algs}
We present a number of algorithms for implementing size $d$
volume-rescaled sampling $\vsdx$ under various assumptions on the
distribution $\dx$. Theorem \ref{l:composition} implies that we can
then construct $\vskx$ by combining $\vsdx$ with an i.i.d.~sample
$\dx^{k-d}$. We can also combine $\vsdx$ with a leverage
score sample $\lev^{k-d}$ or its approximation (see Theorem \ref{t:loss} and Lemma
\ref{l:loss}) to obtain an unbiased estimator with strong loss
bounds. Efficient algorithms for approximate leverage score sampling
were given by \cite{fast-leverage-scores}, as discussed in Section
\ref{s:finite}. Our discussion of volume-rescaled sampling algorithms
starts with the Gaussian random design (Theorem \ref{t:gaussian}). We
then propose a more general algorithm for arbitrary distributions
(Theorem \ref{t:det}), based on a novel idea of
\textit{distortion-free intermediate sampling}, and we adapt it to
some practical settings. Perhaps the most important setting from the
perspective of 
computer science is when distribution $\dx$ is defined as uniform
over a given finite set of $n$ row vectors in $d$ dimensions, where
$n\gg d$. In this case, we improve the time complexity of discrete volume
sampling from $O(nd^2)$ to $O(nd\log n + d^4\log d)$.

\subsection{Volume-rescaled Gaussian distribution}
\label{s:gaussian}
In this section, we obtain a simple formula for producing volume-rescaled
samples when $\dx$ is a centered multivariate Gaussian with any
(non-singular) covariance matrix. We achieve this by making a
connection to the Wishart distribution. The main result follows.
\begin{remark}
For this theorem, given a p.d.~matrix $\A$, we use $\A^{\frac12}$
to denote the unique lower triangular matrix with positive diagonal
entries s.t. $\A^{\frac12}(\A^{\frac12})^\top=\A$.
\end{remark}
\begin{theorem}\label{t:gaussian}
Assume $\dx$ 
    is the normal distribution, i.e., $\x\sim\Nc_d(\zero, \sigd)$. If
  $\X_1 \sim\dx^{k}$ and $\X_2\sim\dx^{k+2}$ are jointly independent,
  then $\X_1(\X_1^\top\X_1)^{-\frac12}(\X_2^\top\X_2)^{\frac12}  \ \sim\ \vskx$.
\end{theorem}
The remainder of Section \ref{s:gaussian} is dedicated to proving
Theorem \ref{t:gaussian}, so we assume that matrix $\X\sim\dxk$
consists of centered $d$-variate normal row vectors with covariance
$\sigd$. Then matrix $\Sigmab=\X^\top\X\in\R^{d\times d}$ is distributed according to
Wishart distribution $W_d(k,\sigd)$ with $k$ degrees of
freedom. The density function of this random matrix is proportional to
$\det(\Sigmab)^{(k-d-1)/2}\exp(-\frac12\tr(\sigd^{-1}\Sigmab))$. On the
other hand, if $\Sigmabb = \Xb^\top\Xb$ is constructed from
$\Xb\sim\vskx$, then its density function is 
multiplied by an additional $\det(\Sigmabb)$, thus increasing
the value of $k$ in the exponent of the determinant. This observation
leads to the following result.
\begin{lemma}\label{t:wishart}
If $\x\sim \Nc_d(\zero,\sigd)$ and $\Xb\sim\vskx$, then
$\Xb^\top\Xb\sim W_d(k+2,\sigd).$
\end{lemma}
\begin{proof}
  Let $\Sigmab=\X^\top\X \sim W_d(k,\sigd)$ and $\Sigmabb\sim
  W_d(k+2,\sigd)$. 
  For any measurable event $A$ over the random matrix $\Xb^\top\Xb$,
  we have
  \begin{align*}
   \Pr\big(\Xb^\top\Xb\!\in\! A\big) &=  \frac{\E[\one_{[\X^\top\X\in
    A]}\det(\X^\top\X)]}{\E[\det(\X^\top\X)]}\\
    &=\frac{\E[\one_{[\Sigmab\in A]}\det(\Sigmab)]}{\E[\det(\Sigmab)]}
      \overset{(*)}{=}\Pr\big(\Sigmabb\!\in\! A\big),
  \end{align*}
where $(*)$ follows because the density function of
Wishart distribution $\Sigmabb\sim W_d(k+2,\sigd)$ is proportional to
$\det(\Sigmabb) \det(\Sigmabb)^{(k-d-1)/2}\exp(-\frac12\tr(\sigd^{-1}\Sigmabb))$.
\end{proof}
This gives us an easy way to produce the total covariance matrix $\Xb^\top\Xb$ of
volume-rescaled samples in the Gaussian case. We next show that the
individual vectors can also be recovered relying on the
following lemma proven in the appendix (Lemma~\ref{l:conditional2}).
\begin{lemma}\label{l:conditional}
    For any $\Sigmab\in\R^{d\times d}$, the conditional distribution
    of $\Xb\sim\vskx$ given $\Xb^\top\Xb=\Sigmab$ is the same as the conditional
    distribution of $\X\sim \dxk$ given $\X^\top\X=\Sigmab$.
  \end{lemma}

  \begin{proofof}{Theorem}{\ref{t:gaussian}}
Let $\Sigmab_1\sim
  W_d(k_1,\sigd)$ and $\Sigmab_2\sim
  W_d(k_2,\sigd)$ be independent Wishart matrices (where
  $k_1+k_2\geq d$). Then matrix
  \[\U = (\Sigmab_1\!+\!\Sigmab_2)^{-\frac12}\Sigmab_1\big((\Sigmab_1\!+\!\Sigmab_2)^{-\frac12}\big)^\top\]
  is matrix variate beta distributed, written as $\U\sim
  B_d(k_1,k_2)$. The following was shown by 
  \cite{matrix-variate-beta}:
  \begin{lemma}[\citeauthor{matrix-variate-beta}, \citeyear{matrix-variate-beta}, Lemma 3.5]\label{l:mvb} 
    If $\Sigmab\sim W_d(k,\sigd)$ is distributed independently of
    $\U\sim B_d(k_1,k_2)$, and if $k=k_1+k_2$, then
    \begin{align*}
      \B &= \Sigmab^{\frac12}\U\big(\Sigmab^{\frac12}\big)^\top
	\quad \text{and} \quad
      \C = \Sigmab^{\frac12}(\I-\U)\big(\Sigmab^{\frac12}\big)^\top
    \end{align*}
    are independently distributed and $\B\sim W_d(k_1,\sigd)$,
    $\C\sim W_d(k_2,\sigd)$.
  \end{lemma}
  Now, suppose that we are given a matrix $\Sigmab\sim
  W_d(k,\sigd)$. We can decompose it into components
  of degree one via
  a splitting procedure described in \cite{matrix-variate-beta},
  namely taking $\U_1\sim B_d(1,k\!-\!1)$ and computing
  $\B_1= \Sigmab^{\frac12}\U_1\big(\Sigmab^{\frac12}\big)^\top$,
  $\C_1=\Sigmab\!-\!\Sigmab_1$ as in
  Lemma~\ref{l:mvb}, then recursively repeating the procedure on
  $\C_1$ (instead of $\Sigmab$) with $\U_2\sim B_d(1,k\!-\!2)$, \ldots,
  until we get $k$  Wishart matrices of degree one summing to $\Sigmab$:
  \begin{align*}
\B_1 &= \Sigmab^{\frac12}\U_1\big(\Sigmab^{\frac12}\big)^\top\\
   \B_2 &=
          \underbrace{\Sigmab^{\frac12}(\I-\U_1)^{\frac12}}_{\text{\normalsize$\C_1^{\frac12}$}}\U_2\underbrace{\big((\I-\U_1)^{\frac12}\big)^\top\big(\Sigmab^{\frac12}\big)^\top}_{\text{\normalsize$\big(\C_1^{\frac12}\big)^\top$}}\\[-3mm]
    \vdots&\\
    \B_k&=\underbrace{\Sigmab^{\frac12}(\I-\U_{k-1})^{\frac12}\dots}_{\text{\normalsize$\C_{k-1}^{\frac12}$}}\U_k\underbrace{\dots \big((\I-\U_{k-1})^{\frac12}\big)^\top\big(\Sigmab^{\frac12}\big)^\top\!\!}_{\text{\normalsize$\big(\C_{k-1}^{\frac12}\big)^\top$}}.
  \end{align*}
  The above collection of matrices can be described more simply via
  the matrix variate Dirichlet distribution. Given independent
  matrices $\Sigmab_i\sim W_d(k_i,\sigd)$ for $i=1..s$, the matrix
  variate Dirichlet distribution $\mathrm{Dir}_d(k_1,\dots,k_s)$ corresponds to a
  sequence of matrices
  \begin{align*}
    \V_i=\Sigmab^{-\frac12}\Sigmab_i\big(\Sigmab^{-\frac12}\big)^\top, \quad  i=1..s, \quad
    \Sigmab=\sum_{i=1}^s\Sigmab_i.
  \end{align*}
  Now,  Theorem 6.3.14 from \cite{matrix-variate-distributions} states
  that matrices $\B_i$ defined recursively as above can 
  also be written as
  \begin{align*}
    \B_i = \Sigmab^{\frac12}\V_i\big(\Sigmab^{\frac12}\big)^\top,\quad
    (\V_1,\dots,\V_k)\sim \mathrm{Dir}_d(1,\dots,1).
  \end{align*}
  In particular, we can construct them as $\B_i=\xbb_i\xbb_i^\top$, where
  \begin{align*}
    \xbb_i=\Sigmab^{\frac12}(\X^\top\X)^{-\frac12}\x_i\quad\text{for}\quad
    \X\sim \dxk.
  \end{align*}
  Note that since matrix $\Sigmab$ is independent of vectors $\x_i$, we
  can condition on it without altering the distribution of the
  vectors. The conditional distribution of matrix
  $\B_i$ determines the distribution of $\xbb_i$ up to multiplying
  by $\pm 1$, and since both $\xbb_i$ and $-\xbb_i$ are identically
  distributed, we conclude that the matrix $\Xb$ formed from rows
  $\xbb_i^\top$ conditioned on $\Xb^\top\Xb=\Sigmab$ has the same
  distribution as $\X$ conditioned on $\X^\top\X=\Sigmab$. So, applying Lemmas
\ref{t:wishart} and  \ref{l:conditional}, if we sample $\Sigmab\sim
  W_d(k+2,\sigd)$, then we obtain $\Xb\sim\vskx$.
\end{proofof}

\subsection{Volume-rescaled sampling for arbitrary distributions}
In this section, we present a general algorithm for volume-rescaled
sampling which uses approximate leverage score sampling to generate a
larger pool of points from which the smaller volume-rescaled sample
can be drawn. The strategy introduced here, called \emph{distortion-free intermediate
  sampling}, has since proven effective for sampling from other
determinantal sampling distributions \citep{dpp-intermediate,dpp-sublinear,alpha-dpp}.
\begin{theorem}\label{t:det} 
Given $\Sigmabh\in\R^{d\times d}$ and i.i.d.~samples from
a $d$-variate distribution $\mathrm{Lev}_{\Sigmabh,{\cal X}}$ such that
  \begin{align}
      (1-\epsilon)\sigd&\preceq\Sigmabh\preceq
    (1+\epsilon)\sigd,&&\text{where}\;
        \epsilon=\frac1{\sqrt{2d}},\label{eq:cond1}\\
   \text{and}\quad \mathrm{Lev}_{\Sigmabh,{\cal
    X}}(A)&\defeq\E_{\dx}\bigg[\one_{[\x^\top\in A]}\frac{\x^\top\Sigmabh^{-1}\x}{\tr(\sigd\Sigmabh^{-1})}\bigg]&&\text{for any event }A,\label{eq:cond2}
  \end{align}
there is an algorithm (Algorithm \ref{alg:det}) which
returns $\Xb\sim \vsdx$, and
with probability at least $1-\delta$ uses
$O(d^2\log\frac1\delta)$ samples from $\hat{L}$ and has time complexity 
    $O(d^4\log\frac1\delta)$.
\end{theorem}
The algorithm relies on a rejection sampling step (line 4) to ensure exact sampling.
Then, to obtain the target sample from the intermediate sample, it
uses ``reverse iterative sampling'' 
\citep{unbiased-estimates-journal} as a subroutine
(see Algorithm~\ref{alg:standard} for a high-level description of this
sampling method). Curiously enough, the efficient implementation
of reverse iterative sampling (not repeated here)
is again based on rejection sampling:
It samples a set of $k$ points out of $n$ in time $O(nd^2)$
(the time complexity is independent of $k$ and holds with high probability).
The key
strength of our sampling method is that it reduces the distribution
$\dx$ to a small sample of $t$ vectors on which the reverse iterative
sampling algorithm is performed. We show that this reduction can be
done efficiently for $t=2d^2$. Even when distribution $\dx$ is a
finite discrete distribution, for example based on a population of $n$
vectors, our algorithm can be used to accelerate reverse iterative
sampling when $n=\Omega(d^2)$. 
\begin{center}
 \begin{minipage}{.55\textwidth}
\begin{algorithm}[H]
  \caption{\small Distortion-free intermediate sampling}
  \label{alg:det}
  \begin{algorithmic}[1]
    \vspace{1mm}
    \STATE \textbf{Input:} $\Sigmabh,\ \mathrm{Lev}_{\Sigmabh,{\cal
        X}},\ t$
    \vspace{1mm}
    \STATE \textbf{repeat}
    \vspace{1mm}
    \STATE \quad $\Xt\leftarrow
    \Big[\!\sqrt{\!
      \frac{d}{\x_i^\top\Sigmabh^{-1}\x_i}}\cdot\x_{i}^\top\Big]_{t\times
      d}$ where $\X\sim \mathrm{Lev}_{\Sigmabh,{\cal
        X}}^t$ \label{line:Xt}
    \STATE \quad Sample $\textit{Acc}\sim
    \text{Bernoulli}\Big(\frac{\det(\frac1t\Xt^\top\Xt)}{\det(\Sigmabh)}\Big)$
    \STATE \textbf{until} $\textit{Acc}=\text{true}$
      \vspace{1mm} 
    \STATE $S
    \leftarrow$ Algorithm \ref{alg:standard} for matrix $\Xt$ and
    $k=d$
    \vspace{1mm}
    \RETURN $\X_S$
    \vspace{1mm}
 \end{algorithmic}
\end{algorithm}
   \end{minipage}\hspace{3mm}
   \begin{minipage}{.42\textwidth}
\begin{algorithm}[H] 
  \caption{\small Reverse iterative sampling \citep{unbiased-estimates-journal}}
  \label{alg:standard}
  \begin{algorithmic}[1]
    \STATE \textbf{Input:} $\X \in\R^{n\times d}$ and $k\geq d$
    \STATE \quad$S \leftarrow \{1..n\}$
\vspace{1mm}
    \STATE \quad{\bf while} $|S|>k$
    \STATE \quad\quad $\forall_{i\in S} \ \ q_i\!\leftarrow\!
    \frac{\det(\X_{S\backslash i}^\top\X_{S\backslash i})}{(|S|-d)\det(\X_S^\top\X_S)}$
    \STATE \quad\quad Sample $i\sim(q_i)_{i\in S}$
\vspace{1mm}
    \STATE \quad\quad $S\leftarrow S \backslash \{i\}$
    \STATE \quad{\bf end} 
    \RETURN $S$
  \end{algorithmic}
\end{algorithm}
   \end{minipage}
\end{center}
\begin{proofof}{Theorem}{\ref{t:det}}
    The distribution $\Lev_{\Sigmabh,{\cal X}}$ integrates to one
  because for $\x^\top\sim \dx$:
  \begin{align*}
    \E\big[\x^\top\Sigmabh^{-1}\x\big]
    = \E\Big[\tr\big(\x\x^\top\Sigmabh^{-1}\big)\Big]
    = \tr\big(\sigd\Sigmabh^{-1}\big).
  \end{align*}
Next, we use the geometric-arithmetic mean
inequality for the eigenvalues of matrix $\frac1t\Xt^\top\Xt\Sigmabh^{-1}$ to show that the
Bernoulli sampling probability is bounded by 1:
\begin{align*}
\frac{\det\!\big(\frac1t\Xt^\top\Xt\big)}{\det\!\big(\Sigmabh^{-1}\big)}
  &\leq\Big(\frac{1}{d\,t}\tr\big(\Xt^\top\Xt\Sigmabh^{-1}\big)\Big)^{\!d}
  =\Big(\frac{1}{d\,t}
    \sum_{i=1}^t\frac{d}{\x_i^\top\Sigmabh^{-1}\x_i}\tr\big(\x_i\x_i^\top\Sigmabh^{-1}\big)\Big)^d=1.
\end{align*}
Let $\xbt^\top\sim D_{\cal\widetilde{\!X}}$ be distributed
as a row vector of $\Xt$ as sampled in line \ref{line:Xt}.
The distribution of matrix $\Xt$ returned by rejection sampling 
after exiting the \textbf{repeat} loop changes to:
\begin{align*}
\E_{D_{\cal
  \widetilde{\!X}}^t}\!\bigg[\one_{[\Xt\in A]}
  \frac{\det(\frac1t\Xt^\top\Xt)}{\det(\Sigmabh)}\bigg] 
  \propto \E_{D_{\cal\widetilde{\!X}}^t}\!\Big[\one_{[\Xt\in A]}
  \det\!\big(\Xt^\top\Xt\big)\Big]
  \propto\Vol_{D_{\cal  \widetilde{\!X}}}^t(A),
  \end{align*}
  i.e., volume-rescaled sampling from $D_{\cal \widetilde{\!X}}$.
Now Theorem \ref{l:composition} implies that
$\Xt_S\sim\Vol_{D_{\cal  \widetilde{\!X}}}^d$. In particular, it
means that the distribution of $\X_S$ is the
same for any choice of $t\geq d$. We use this observation to compute
the probability of an event $A$ w.r.t.~sampling of
$\X_S$ (up to constant factors) by setting $t=d$:
\begin{align*}
  \Pr(A) &\propto\E_{\dx^d} \bigg[\,\one_{[\X\in A]}\,\det\!\Big(\frac1t\Xt^\top\Xt\Big)
\cdot\prod_{i=1}^d\x_i^\top\Sigmabh^{-1}\x_i\bigg]\\
  &\overset{(*)}{=} \E_{\dx^d} \bigg[\,\one_{[\X\in A]}\,\frac{\det(\X^\top\X)}{(\frac{d}{t})^d
    \prod_i\x_i^\top\Sigmabh^{-1}\x_i}\cdot\prod_{i=1}^d\x_i^\top\Sigmabh^{-1}\x_i\bigg]\\
       &\propto \E_{\dx^d}
         \big[\,\one_{[\X\in A]}\,\det(\X^\top\X)\big]\\
  &\propto\vsdx(A),
\end{align*}
where $(*)$ uses the fact that for $t=d$, $\det(\Xt^\top\Xt)=\det(\Xt)^2$ is the
squared volume of the parallelepiped spanned by the rows of $\Xt$.
Thus, we established the correctness of Algorithm~\ref{alg:det} for
any $t\geq d$, and we move on to complexity analysis. If we think of each iteration of the
\textbf{repeat} loop as a single Bernoulli trial, the success
probability $\Pr(\textit{Acc}\!=\!\text{true})$ equals
$\E[\det(\frac1t\Xt^\top\Xt)/\det(\Sigmabh)]$ where $\Xt\sim D_{\cal
  \widetilde{\!X}}$. Note that
\begin{align*}
  \E\big[\Xt^\top\Xt\big] = \sum_{i=1}^t\E\bigg[\frac
  d{\x_i^\top\Sigmabh^{-1}\x_i}\x_i\x_i^\top\bigg]
  =\sum_{i=1}^t\frac{d}{\tr(\sigd\Sigmabh^{-1})}\sigd = \frac{d\,t}{\tr(\sigd\Sigmabh^{-1})}\sigd.
\end{align*}
So, using Lemma~\ref{l:determinant} on the matrix $\Xt$ we obtain that:
\begin{align*}
  \E\bigg[\frac{\det(\frac1t\Xt^\top\Xt)}{\det(\Sigmabh)}\bigg]
  &=\frac{(t^{\underline{d}}/t^d)\cdot\det(\frac1t\E[\Xt^\top\Xt])}{\det(\Sigmabh)}
 = \frac{(t^{\underline{d}}/t^d)\cdot\det(\sigd)}
    { (\frac{1}{d}\tr(\sigd\Sigmabh^{-1}))^d\det(\Sigmabh)}\\
  &=\bigg(\prod_{i=0}^{d-1}\frac{t-i}{t}\bigg)
    \frac{\det(\sigd\Sigmabh^{-1})}
    {(\frac{1}{d}\tr(\sigd\Sigmabh^{-1}))^d}
  \geq \bigg(1-\frac{d}{t}\bigg)^d \frac{\det(\sigd\Sigmabh^{-1})}
    {(\frac{1}{d}\tr(\sigd\Sigmabh^{-1}))^d}.
\end{align*}
Let $\lambda_1,\dots,\lambda_d$ be the
eigenvalues of matrix $\Sigmabh \sigd^{-1}$. The approximation
guarantee for $\Sigmabh$ implies that all of these eigenvalues lie in
the range $[1\!-\!\epsilon,1\!+\!\epsilon]$. To lower-bound the
success probability, we use the Kantorovich arithmetic-harmonic mean
inequality. Letting $A(\cdot)$, $G(\cdot)$ and $H(\cdot)$ denote the
arithmetic, geometric and harmonic means respectively:
\begin{align*}
\frac{\det(\sigd\Sigmabh^{-1})}{(\frac{1}{d}\tr(\sigd\Sigmabh^{-1}))^d}
  &=\frac{\prod_{i=1}^d\frac{1}{\lambda_i}}{(\frac{1}{d}\sum_{i=1}^d\frac{1}{\lambda_i})^d}
  =\bigg(\frac{H(\lambda_1,\dots,\lambda_d)}{G(\lambda_1,\dots,\lambda_d)}\bigg)^d\\
  &\overset{(a)}{\geq}\bigg(\frac{H(\lambda_1,\dots,\lambda_d)}
    {A(\lambda_1,\dots,\lambda_d)}\bigg)^d
  \overset{(b)}{\geq}\big((1\!-\!\epsilon)(1\!+\!\epsilon)\big)^d
  \ =\ \Big(1-\frac{1}{2d}\Big)^d
\end{align*}
since $\epsilon=\frac1{2\sqrt{d}}$, where $(a)$ is the
geometric-arithmetic mean inequality and $(b)$ is 
the Kantorovich inequality \citep{kant} with
$a=1-\epsilon$ and $b=1+\epsilon$:
$$\text{For}\quad
0< a\le \lambda_1 , \mydots, \lambda_d \le b,\quad 
\frac{A(\lambda_1,\!\mydots,\lambda_d)}{H(\lambda_1,\!\mydots,\lambda_d)}
\leq \bigg(\frac{A(a,b)} {G(a,b)}\bigg)^{\!2}.
$$
Now setting $t=2d^2$ we obtain the following lower bound for the acceptance probability:
\begin{align*}
\Pr(\textit{Acc}\!=\!\text{true}) =
  \E\bigg[\frac{\det(\frac1t\Xt^\top\Xt)}{\det(\Sigmabh)}\bigg]
  \geq \Big(1-\frac{1}{2d}\Big)^{2d}\geq\frac14.
\end{align*}
So a simple tail bound on a geometric random variable shows
that the number of iterations of the \textbf{repeat} loop is $r\leq
\ln(\frac1\delta)/\ln(\frac43)$ w.p. at least $1-\delta$.
We conclude that the number of samples needed from
$\Lev_{\Sigmabh,{\cal X}}$ is $O(d^2\log\frac1\delta)$ w.p. at least
$1-\delta$. Note that the computational cost per sample is $O(d^2)$
and the cost of Algorithm \ref{alg:standard} is $O(d^4)$, obtaining
the desired complexities. 
\end{proofof}
\subsection{Distributions with bounded support}
Theorem \ref{t:det} requires some knowledge about the distribution
$\dx$, namely the approximate covariance matrix $\Sigmabh$ and
i.i.d.~samples from an approximate leverage score distribution
$\levh$. In this and the following section we show that these can be
computed efficiently in certain standard settings. For this section, suppose
that distribution $\dx$ has bounded support. We use a standard 
notion of \emph{conditioning number} for multivariate distributions
\citep[see, e.g.,][]{chen2017condition}.
\begin{definition}
Let $\dx$ be a $d$-variate distribution with bounded support set
$\mathrm{supp}(\dx)\subseteq\R^{1\times d}$. The conditioning number
$\kd$ of this distribution is defined as:
  \begin{align*}
  \kd \defeq \sup_{\xbt\in\mathrm{supp}(\dx)} \xbt^\top\sigd^{-1}\xbt.
\end{align*}
\end{definition}
We next show that when the conditioning number $\kd$ is bounded by
some known constant $K$, then all input arguments of Algorithm
\ref{alg:det} can be computed from a small number of independent draws from
$\dx$. In the following result the term \textit{sample complexity}
refers to the number of i.i.d.~samples from $\dx$ used by an algorithm.
\begin{theorem}
Suppose that $\kd\leq K$. Then
for any $\delta\in(0,1)$ and positive integer $c$, there is an algorithm with sample complexity
$O(cKd\log d/\delta)$ and time complexity $O(cKd^3\log d/\delta)$ which succeeds w.p.~at
least $1-\delta$ and returns a matrix
$\Sigmabh$ satisfying \eqref{eq:cond1} and $\X\sim\levh^{cd^2}$.
\end{theorem}
\begin{proof}
Setting $\epsilon=\frac1{\sqrt{2d}}$ in
Lemma~\ref{l:matrix-tail}, we
observe that the sample complexity of obtaining $\Sigmabh$ with desired accuracy
is  $m=O(\kd d\log d/\delta)$, and computing it takes $O(md^2) =
  O(\kd d^3\log d/\delta)$. Sampling from $\levh$ can be done via
  rejection sampling as follows:
  \begin{align*}
    \x^\top\sim\dx,\qquad
    \text{acc}\sim\text{Bernoulli}\Big((1-\epsilon)\cdot\x^\top\Sigmabh^{-1}\x\,/ K\Big).
  \end{align*}
 We can lower bound the acceptance probability as follows:
\begin{align*}
  \Pr(\text{acc}\!=\!\text{true})&=(1-\epsilon)\cdot\E\bigg[\frac{\x^\top\Sigmabh^{-1}\x}{K}\bigg]
  = (1-\epsilon)\frac{\tr(\sigd\Sigmabh^{-1})}{K}
    \geq\frac{1-\epsilon}{1+\epsilon}\cdot\frac{d}{K}. 
\end{align*}
We conclude that with probability at least $1-\delta$ the number of
samples from $\dx$ needed to obtain 
$cd^2$ samples from $\levh$ is $O(cd^2
(K/d)\log 1/\delta)=O(cKd\log 1/\delta)$. Computing each acceptance
probability takes $O(d^2)$, which concludes the proof.
\end{proof}
\subsection{Sampling from finite datasets}\label{s:finite}
For this section we assume that $\dx$ is a uniform distribution over a
set of $n\gg d$ vectors $\{\x_1,\dots,\x_n\}$. In this case,
the distribution $\vsdx$ corresponds to sampling a set $S\subseteq [n]$ of
size $d$ such that $\Pr(S)\propto \det(\X_S)^2$, i.e., discrete volume
sampling. The input arguments
for Algorithm \ref{alg:det} can be computed efficiently using standard
sketching techniques, which leads to the first algorithm for discrete
volume sampling that (for large enough $n$) runs in time $o(nd^2)$.
\begin{theorem}\label{t:finite}
  Let $\X\in\R^{n\times d}$ be a fixed matrix. For any $\delta>0$
  there is an algorithm with time complexity $O(nd\log n +
  d^4\log d)\cdot\poly\log 1/\delta$ that succeeds w.p.~at least  $1-\delta$, and then returns
  a random set $S\subseteq [n]$ of size $d$ such that $\Pr(S)\propto
  \det(\X_S)^2$. 
\end{theorem}
\begin{proof}
Naturally it suffices to show that the inputs for Algorithm
  \ref{alg:det} can be constructed efficiently. First note that
  $\sigd=\frac1n\X^\top\X$, and we can compute an
  $\epsilon$-approximation $\Sigmabh$ of this matrix
in time $O(nd\log n + d^3\epsilon^{-2}\log d)$, where $\epsilon=\frac1{2\sqrt{d}}$, using a sketching
technique called Fast Johnson-Lindenstraus Transform~\citep{ailon2009fast}, as described in
\cite{fast-leverage-scores}. Now, we need to produce samples from the
leverage score-type distribution $\levh$, which in this setting
corresponds to a discrete distribution over the index set $[n]$.
Using a different sketch of the data, an
approximation $\hat{L}=(\hat{L}_1,\dots,\hat{L}_n)$ of this
distribution can be computed in time $O(nd\log n + d^3)$ as shown in
\cite{fast-leverage-scores}, which satisfies 
$\hat{L}_i\geq \frac{\x_i^\top\Sigmabh^{-1}\x_i}{2\cdot\tr(\sigd\Sigmabh^{-1})}$.
Then we can use rejection sampling to get i.i.d.~samples from
$\levh$. All of the above randomized procedures succeed w.p.~at
least $1-\delta$, where the time complexity scales with $\poly\log
1/\delta$. Conditioned on them succeeding, Algorithm \ref{alg:det}
samples exactly from the distribution $\vsdx$ in time $O(d^4)\cdot
\poly\log1/\delta$, concluding the proof.
\end{proof}

\section{Experiments}
\label{s:experiments}
Subsampling from large datasets is an important practical application
of our methods. In this context, distribution $\dxy$ is defined via a
fixed matrix $\X\in\R^{n\times d}$ and a vector $\y\in\R^n$ by
sampling a row-response pair $(\x_i^\top,y_i)$ uniformly at
random. The square loss for this problem becomes
$L_{\dxy}(\w)=\frac{1}{n}\|\X\w-\y\|^2$. A commonly used
approach in this problem is \textit{leverage score sampling} \citep{drineas2006sampling}.
In Section \ref{s:loss} we propose a hybrid sampling scheme which
combines leverage score sampling with volume-rescaled sampling. We
will call it here \textit{leveraged volume sampling}. As discussed in
Section~\ref{s:algs}, this method can be implemented very efficiently
(see also Figure \ref{fig:runtime}), with time complexity similar to
leverage score sampling. In the following experiments we evaluate the
loss $L_{\dxy}$ of the estimators produced by both methods, showing
that if the sample size is small, then leveraged volume sampling
performs significantly better than leverage score sampling. We also contrast this with the
estimators produced by a previously proposed variant of discrete volume
sampling, given by \cite{unbiased-estimates-journal}, which for larger
sample sizes does not perform as well as the other two methods.
Overall, the three estimators we tested are:
\begin{align*}
  \textit{volume sampling:} \quad\wbh
  &= (\X_S)^\dagger\y_S,&\Pr(S)&\sim
  \det(\X_S^\top\X_S),
\quad S\in{[n]\choose k},\\[-2mm]
  \textit{leverage score sampling:}\quad
  \wbh&=  (\P_{\!\!\Xh}\Xh)^\dagger\P_{\!\!\Xh}\ybh,
  &\Xh&\sim\lev^k,\quad
    \P_{\!\X}=\sum_{i=1}^k
    \tfrac1{\sqrt{l_{\x_i}}}\e_i\e_i^\top,
  \\
\textit{leveraged volume sampling:}\quad
 \wbh &=(\P_{\!\!\Xb}\Xb)^\dagger\P_{\!\!\Xb}\ybb,
&\Xb &\sim\vsdx\cdot \lev^{k-d}.
\end{align*}
For the latter two estimators, the response vector is constructed from
$D_{\cal Y|\x}$, i.e., to match the selected row vectors. Both the
volume sampling-based estimators are unbiased, however
the leverage score sampling estimator is not. The volume
sampling method proposed in prior work is very similar to our distribution $\vskx$ defined
w.r.t.~uniform sampling from the dataset, except for the 
fact that the former does not allow the same row from the dataset to appear more
than once in the sample (because $S$ is a set). For large datasets
that difference does not have any practical impact on the
estimator. In particular, as discussed in Section \ref{s:lower}, our
lower bound from Theorem \ref{t:lower-vskx} can be easily adapted to hold
for this method as well. 

\begin{wrapfigure}{l}{0.5\textwidth}
\begin{tabular}{c|c|c}
Dataset & Instances ($n$) & Features ($d$) \\
\hline
\textit{bodyfat} & 252& 14\\
\textit{cpusmall} & 8,192 &12\\ 
\textit{mg} & 1,385 & 21\\
\textit{abalone} & 4,177  & 36 \\
\textit{cadata} & 20,640&8\\
\textit{MSD} &463,715&90
\end{tabular}
\captionof{table}{Libsvm regression datasets \citep{libsvm}. We
  expanded the features in \textit{mg} and \textit{abalone} to all
  degree 2 monomials, and removed redundancies.}
\label{tab:datasets}
  \includegraphics[width=.5\textwidth]{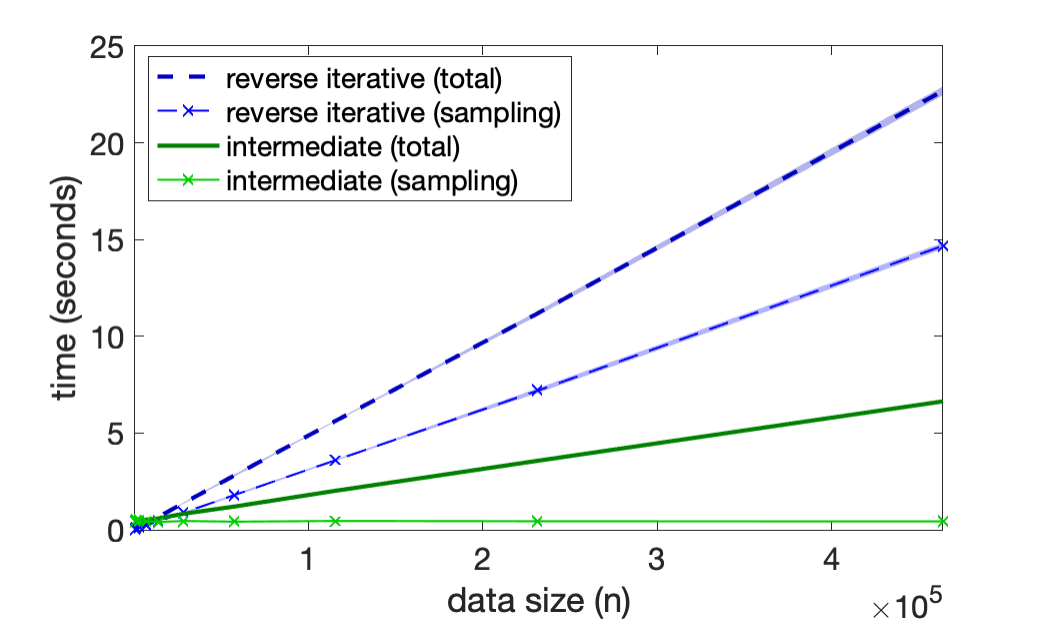}
  \captionof{figure}{Runtime comparison of algorithms for discrete
    volume sampling on the MSD dataset, varying
    the data size $n$ by taking row subsets of the full data matrix.}
  \label{fig:runtime}
\end{wrapfigure}

For each
estimator we plotted the loss $L_{\dxy}(\wbh)$ for a range of sample
sizes $k$, contrasted with the loss of the  best least-squares estimator $\w^*$
computed from all data. 
Plots shown in Figure \ref{fig:experiments} were
averaged over 100 runs, with shaded area representing standard error
of the mean. We used six benchmark datasets from the libsvm
repository \citep{libsvm}, whose dimensions are given in Table
\ref{tab:datasets}.

The results confirm that our proposed leveraged
volume sampling is as good or better than either of the baselines for any sample size
$k$. We can see that, in some of the examples, standard volume sampling
exhibits bad behavior for larger sample sizes, as suggested by the
lower bound of Theorem \ref{t:lower-vskx} (especially noticeable on
\textit{bodyfat} and \textit{cpusmall} datasets). On the other hand, leverage
score sampling exhibits poor performance for small sample sizes due to
the coupon collector problem, which is most noticeable for
\textit{abalone} dataset, where we can see a very sharp transition 
after which leverage score sampling becomes effective. Neither of the
variants of volume sampling suffers from this issue.

Finally, in Figure \ref{fig:runtime}, we compared the computational
cost of implementing discrete volume sampling using our
new distortion-free intermediate sampling (Algorithm~\ref{alg:det}) to
the prior state-of-the-art method of \cite{unbiased-estimates-journal}, reverse iterative sampling
(Algorithm~\ref{alg:standard}). Note that the output samples from the two
algorithms are identically distributed according to $\vsdx$, where $\dx$ denotes
the uniform distribution over the dataset, and both of the volume sampling
distributions considered in our experiments can be implemented using either
of these algorithms. In the figure, we distinguished between the
``total'' cost and ``sampling'' cost: the sampling cost excludes any
preprocessing steps that can be avoided during repeated sampling (see
Section \ref{s:app} for the motivations of repeated volume sampling).
The preprocessing cost for both methods involves computing the leverage
scores of the data matrix. The experiments were performed on MSD, the
largest dataset considered in this empirical evaluation. We varied the
data size by taking subsets of the full data matrix. The results were averaged over 5 runs,
with the shaded area representing standard deviation.
For the total cost, Figure \ref{fig:runtime} shows that both methods
scale linearly with $n$, however 
our intermediate sampling approach is considerably faster for large
data sizes, up to a factor of 3 in this experiment. When we look at the
sampling cost, the gap between the two approaches becomes much larger
because the cost of reverse iterative sampling still grows linearly
with $n$, whereas the cost of intermediate sampling stays flat. As a
result, for the full MSD dataset we observe at least an order of
magnitude difference. This is consistent with our analysis, since
Algorithm~\ref{alg:det} effectively reduces the dataset
down to an intermediate sample with size independent of $n$, and then
runs reverse iterative sampling on that intermediate sample. Thus, the
vast majority of the total cost of intermediate sampling involves
the preprocessing step of computing the leverage scores. It is worth noting that for even larger datasets, further
computational savings in the preprocessing cost can be achieved by computing the leverage scores
approximately (see Section \ref{s:finite}).

\begin{figure}
  \includegraphics[width=0.5\textwidth]{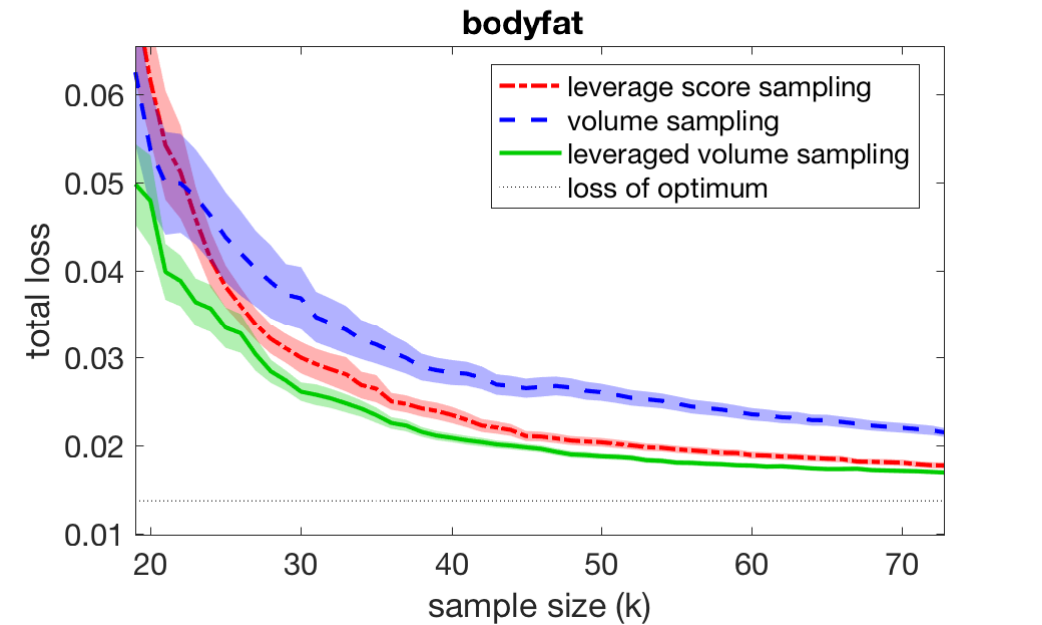}\nobreak
  \includegraphics[width=0.5\textwidth]{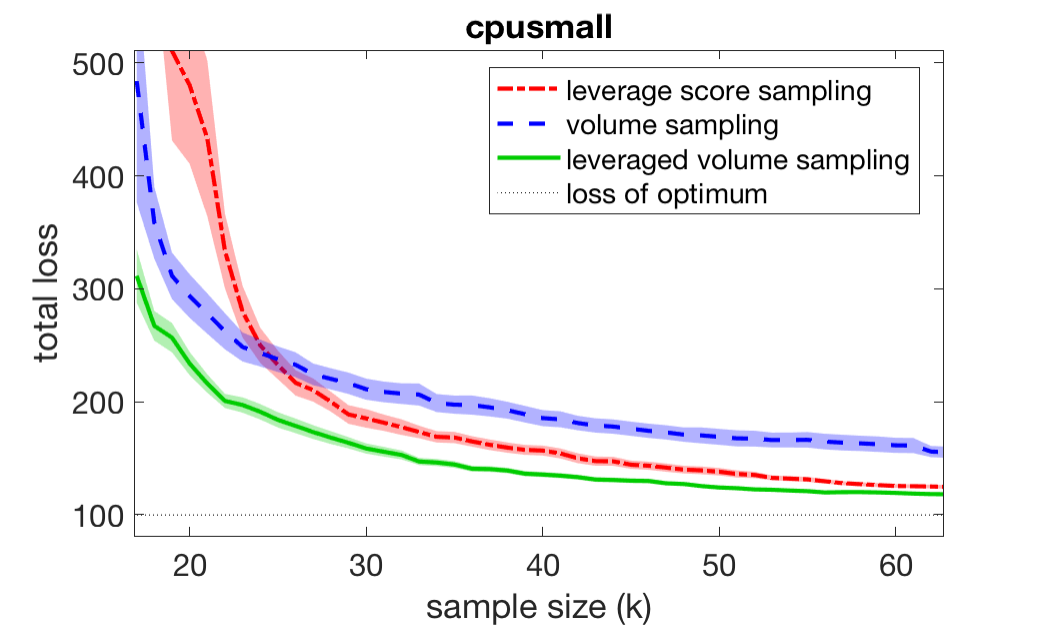}
\includegraphics[width=0.5\textwidth]{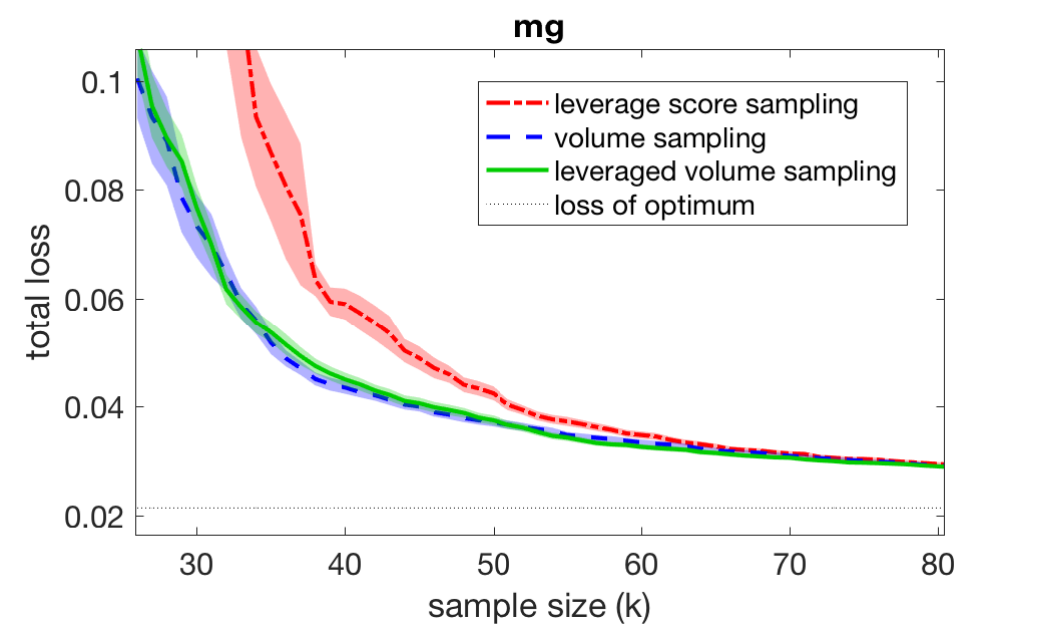}\nobreak
\includegraphics[width=0.5\textwidth]{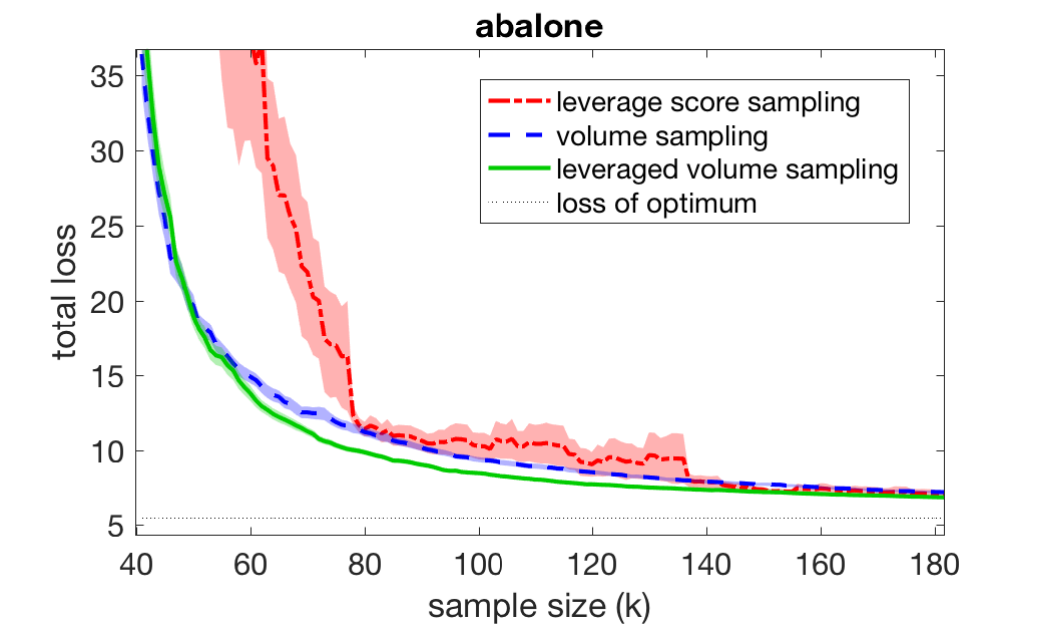}
\includegraphics[width=0.5\textwidth]{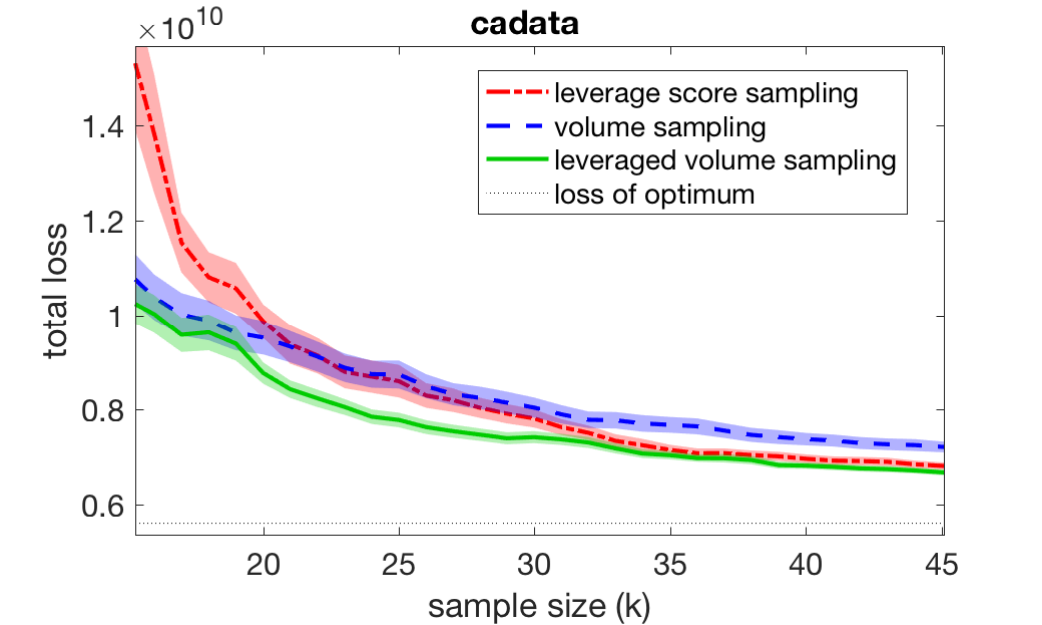}\nobreak
\includegraphics[width=0.5\textwidth]{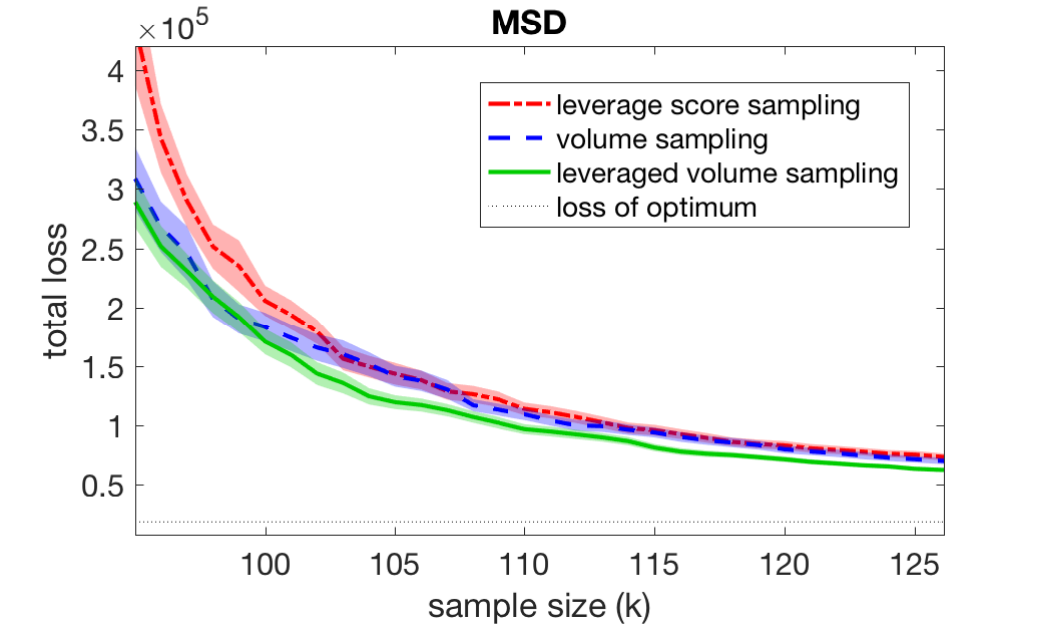}
\caption{Comparison of loss of the subsampled estimator when
  using \textit{leveraged volume sampling} vs using \textit{leverage score sampling} and
  standard \textit{volume sampling}  on six datasets.}
\label{fig:experiments}
\end{figure}

\section{Conclusions} \label{s:conclusion}

We showed that
for any input distribution and $\epsilon>0$, 
there is a random design consisting of
$O(d\log d+ d/\epsilon)$ points from which an {\em unbiased} estimator can
be constructed whose expected square loss over  
the entire distribution is bounded by $1+\epsilon$ times
the loss of the optimum. 
However, two main open problems remain.
First, can the sample size bound be reduced to $O(d/\epsilon)$?
This has already been done with a {\em biased estimator}
by \cite{chen2017condition}, but finding an {\em unbiased} estimator
of the smaller size remains open. 

Second, the least squares estimator combined with i.i.d.~leverage
score sampling already achieves loss $1+\epsilon$ times the optimum
with $O(d\log d+ d/\epsilon)$ points. 
The resulting estimator is biased. However, in our
preliminary experiments the bias of
\emph{exact} leverage score sampling is small and decreases
rather quickly (unlike for uniform sampling, or even approximate
leverage score sampling, where the bias can
be significant). Thus, one of the key open problems is
to quantify the bias of this method.

\acks{Micha{\l } Derezi\'{n}ski and Manfred K. Warmuth
  were supported by the NSF grant IIS-1619271. Micha{\l }
  Derezi\'{n}ski would also like to thank the NSF for funding via the
  NSF TRIPODS program. Daniel Hsu was supported by the NSF grant
  CCF-1740833 and a Sloan Research Fellowship. Part of this work was done
while Manfred K. Warmuth was visiting Google Inc.~in Z\"urich and 
Micha{\l } Derezi\'{n}ski was visiting the Simons Institute for the Theory of
Computing. We would also like to acknowledge Eric Price for valuable discussions
regarding this paper.}

\appendix
\section{Exact calculation of $\w^*$ for the 
i.i.d.\ Gaussian experiment of the introduction
and a technical proposition}
\label{a:exact}

Since in the setup $\sigd=\I$,
the least squares solution can be computed as:
\begin{align*}
  \w^* &= \argmin_\w\E\big[(\x^\top\w-y)^2\big]= \sigd^{-1}\E[y\x] \\
  &= \sum_{i=1}^d\E\Big[\big(\tfrac13x_i^3 + x_i\big)\x\Big] =
\begin{pmatrix}
\E[\frac13 x_1^4+x_1^2]\\ |\\ \E[\frac13 x_d^4+x_d^2]
\end{pmatrix} = \begin{pmatrix}2\\ |\\ 2\end{pmatrix}.
\end{align*}
Here the second to last equality uses the fact that the
cross terms are 0 due to independence and the last equality
follows from the fact that $\E[x^4]=3$ and $\E[x^2]=1$, for $x \sim \Nc(0,1)$.

\begin{proposition}[Theorem 2 of~\citealp*{cho2009inner}] \label{prop:uniform-moments}
  Let $\u = (u_1,\dotsc,u_d)$ be a uniformly random unit vector in $\R^d$.
  For any $k_1,\dotsc,k_d \geq 0$,
  \[
    \E\bigg[ \prod_{j=1}^d |u_j|^{2k_j} \bigg]
    = \frac{\prod_{j=1}^d \Gamma\big( k_j + \tfrac12 \big)}{\Gamma\big( \sum_{j=1}^d k_j + \tfrac{d}{2} \big)} \cdot \frac{\Gamma\big( \tfrac{d}{2} \big)}{\Gamma\big( \tfrac12 \big)^d} .
  \]
\end{proposition}

\section{Loss bound with approximate leverage scores}
\label{a:loss}
In this section we describe the changes needed for the proof of
Theorem \ref{t:loss} to be extended to approximate leverage score
sampling, as described in Lemma \ref{l:loss}. Below, we state the
result in its full generality. Recall that we denote a leverage score
of point $\x$ as $l_\x=\x^\top\sigd^{-1}\x$.
\begin{theorem}\label{t:loss-full}
  Let $\dx$ be a $d$-variate distribution. Assign
  to every $\x^\top\!\in\mathrm{supp}(\dx)$ a real-valued $\hat{l}_\x$
  such that $\frac12 l_\x\leq \hat{l}_\x\leq \frac32 
  l_\x$ and define the following $d$-variate distribution:
  \begin{align*}
    \levhh(A) \defeq \frac{\E_{\dx}\big[\one_{[\x^\top\in
    A]}\,\hat{l}_\x\big]}{\E_{\dx}\big[\hat{l}_\x\big]}\quad\text{for
any    $\dx$-measurable $A$.}
  \end{align*}
  For any $\epsilon >0$, there
   is  $k=O(d\log d+d/\epsilon)$ such that for any $D_{{\cal Y}|\x}$, if we sample $\Xb\sim \vsdx\!\cdot \levhh^{k-d}$
   and $\yb_i\sim D_{{\cal Y}|\x=\xbb_i}$ then $\wbh =
   \argmin_\w \sum_{i=1}^k \frac1{\hat l_{\xbb_i}}(\xbb_i^\top\w -
   \yb_i)^2$ satisfies: 
   \begin{align*}
\quad \E[\wbh] &= \argmin_\w L_{\dxy}(\w) \quad\text{and}\quad \E\big[L_{\dxy}(\wbh)\big] \leq
     (1+\epsilon)\cdot \min_\w L_{\dxy}(\w).
     \end{align*}
 \end{theorem}
 \begin{proof}
The reduction described at the beginning of the proof of Theorem
\ref{t:loss} proceeds almost unchanged, except that now distribution
$\dxyt$ is defined in terms of the approximate leverage scores:
\begin{align*}
(\xbt^\top\!,\yt)=\bigg(\frac1{\sqrtlh{\xbh}}\,\xbh^\top\!,
  \frac1{\sqrtlh{\xbh}}\,\yh\bigg)\sim\dxyt,
\end{align*}
where $\xbh\sim\levhh$ and $\yh\sim D_{{\cal Y}|\x=\xbh}$. Denoting
$\hat{d}=\E_{\dx}[\hat{l}_\x]\in[\frac12d,\frac32d]$, we have
$\Sigmab_{\dxt}\!=\sigd/\hat{d}$. Also, the
leverage scores of $\dxyt$ are approximately uniform:
\begin{align*}
\xbt^\top\Sigmab_{\dxt}^{-1}\xbt =\frac1{\hat{l}_{\xbh}}\xbh^\top\Sigmab_{\dxt}^{-1}
  \xbh=\frac{\hat{d}}{\hat{l}_{\xbh}}\xbh^\top\sigd^{-1}\xbh\,\in\,
  [d/3,\,3d].
\end{align*}
Following the same steps as for Theorem \ref{t:loss}, we conclude that
without loss of generality it suffices to show the result w.r.t.~loss
$L_{\dxyt}$ for the estimator $\Xt^\dagger\ybt$ drawn 
from $\Xt\sim\Vol_{\dxt}^d\!\!\cdot\dxt^{k-d}$ and $\yt_i\sim \dxyt_{{\cal Y}|\xbt=\xbt_i}$.

Using the above reduction, from now on we assume that $l_\x\in[d/3,\,3d]$ a.s.~for $\x\sim\dx$,
and we consider the estimator $\Xb^\dagger\ybb$, where $\Xb\sim\vsdx\cdot\dx^{k-d}$.
Now, the unbiasedness of this estimator follows immediately from Theorems
 \ref{l:composition} and \ref{t:unbiased}. Again, without loss of generality, we can replace the distribution
$\x^\top\sim\dx$ by the distribution of $\x^\top\sigd^{-\sfrac12}$,
so from now on we will let $\sigd=\I$. The loss bound reduces to
 the following, same as before:
\begin{align}
L_{\dxy}(\wbh)-L_{\dxy}(\w^*)
  &= \|\wbh-\w^*\|^2\leq \big\|(\Xb^\top\Xb)^{-1}\|^2\cdot\|\Xb^\top(\ybb-\Xb\w^*)\big\|^2.\label{eq:terms-full}
\end{align}
Applying Lemma \ref{l:matrix-tail} for $\dx$ with $K=3d$, $m=k-\lfloor k/2\rfloor$
and $\epsilon=1/2$ we obtain that if $k\geq d+12Cd\log d/\delta$
then $\Xb\sim\vsdx\cdot\dx^{k-d}$ with probability at least $1-\delta$ satisfies
\begin{align*}
  \Ec:\qquad\Xb_{[s]^c}^\top\Xb_{[s]^c}\succeq\frac k4\cdot\I,\quad
  \text{ where }s=\lfloor k/2 \rfloor.
\end{align*}
We now decompose the expectation into two terms depending on
whether the event $\Ec$ occurs or not:
\begin{align}
\E[\|\wbh-\w^*\|^2]= \Pr(\Ec)\cdot\E[\|\wbh-\w^*\|^2\mid \Ec] +
  \Pr(\neg\Ec)\cdot\E[\|\wbh-\w^*\|^2\mid \neg\Ec],
\end{align}
and the proof is divided into two parts, for handling the two terms.

\paragraph{Part 1: Event $\Ec$ suceeds} We use the upper bound from \eqref{eq:terms-full}.
Event $\Ec$ implies that
$\|(\Xb^\top\Xb)^{-1}\|^2\leq 4^2/k^2$. The second term in
\eqref{eq:terms-full} is decomposed similarly as in 
\eqref{eq:marginals0}, however bounding each of the obtained
components will require a bit more care. Denoting $\rbb=\ybb-\Xb\w^*$,
we have
\begin{align*}
  \E\big[\big\|\Xb^\top\rbb\big\|^2\big]
  &=\sum_{\{i,j\}\subseteq[d]}\!\!\E\big[\rb_i\rb_j\xbb_i^\top\xbb_j\big]
    +\sum_{i\in[d]}\E\big[\|\xbb_i\rb_i\|^2\big]
    +\sum_{i\in[d]^c}\E\big[\|\xbb_i\rb_i\|^2\big]
  \\ &=
       d(d\!-\!1)\,\E\big[\rb_1\rb_2\xbb_1^\top\xbb_2\big]
       + d\,\E\big[\rb_1^2 l_{\xbb_1}\big] + (k-d)\,\E_{\dx}\big[(y-\x^\top\w^*)^2l_\x\big].
\end{align*}
Since $l_\x\leq 3d$, the last component above can be immediately
bounded by $3d(k-d)L_{\dxy}(\w^*)$. Invoking Theorem \ref{t:marginal},
we know that $\xbb_1\sim\lev$ so the second term can be bounded as
follows:
$d\,\E[\rb_1^2l_{\xbb_1}]=d\,\E_{\dxy}[(y-\x^\top\w^*)l_\x^2]/d\leq
9d^2L_{\dxy}(\w^*)$. The remaining term is computed by invoking Lemma
\ref{l:joint}. Denoting $r_i=y_i-\x_i^\top\w^*$, we have
  \begin{align*}
    d(d\!-\!1)\,\E\big[\rb_1\rb_2\xbb_1^\top\xbb_2\big]
    &= d(d\!-\!1)\,\E_{\dxy^2}\big[r_1r_2\x_1^\top\x_2\cdot
      \big(l_{\x_1}l_{\x_2}\
-(\x_1^\top\x_2)^2\big)\big]/d^{\underline{2}}
\\ &= \big\|\E_{\dxy}[(y-\x^\top\w^*)l_\x \x]\big\|^2-
     \underbrace{\E_{\dxy^2}\big[r_1r_2(\x_1^\top\x_2)^3\big]}_{\geq
     0}
\\[-5mm] &\overset{(*)}{\leq} \E_{\dxy}\big[(y-\x^\top\w^*)l_\x^2\big]\leq 9d^2L_{\dxy}(\w^*),
  \end{align*}
where $(*)$ is implied by the following more general property
  of the random vector $\x^\top\sim\dx$ when $\sigd=\I$: for any random variable $b$ jointly
  distributed with $\x$ we have $\|\E[b\,\x]\|^2\leq\E[b^2]$. This
  follows because $\E[\x\x^\top] =\I$, so the
  components of $\x$, treated as scalar random variables, form an
  orthonormal basis of a $d$-dimensional subspace of the Hilbert space $\Hc$
  of square-integrable random variables. Thus, $\|\E[b\x]\|^2$,
  which is the $\Hc$-norm of the projection of $b$ onto that subspace,
  is no more than the $\Hc$-norm of $b$ itself.

  \paragraph{Part 2: Event $\Ec$ fails} This part follows identically
  as in the proof of Theorem \ref{t:loss}, except that when applying
  Lemma \ref{l:decorrelation}, we use the fact that $\|\x\|^2\leq 3d$,
  obtaining:
  \begin{align*}
    \E\big[\tr((\Xb_{[s]}\Xb_{[s]})^{-1})\|\rbb_{[s]}\|^2\big]
    &\leq s\cdot\Big(\frac ds\cdot L_{\dxy}(\w^*) + \frac{d-1}{s(s-d+1)}
   \cdot \E_{\dx}\big[\|\x\|^2\rb_1^2\big]\Big)
    \\
    &\leq d\cdot L_{\dxy}(\w^*) + \frac{3d(d-1)}{s-d+1}\cdot
L_{\dxy}(\w^*)\leq 10d\,L_{\dxy}(\w^*).
  \end{align*}
With the remaining steps same as in Theorem \ref{t:loss}, this concludes
the proof.
\end{proof}

\section{Volume-rescaled sampling conditioned on the
  covariance}\label{a:conditional}
In this section we present the proof of a lemma used to construct
volume-rescaled samples when $\dx$ is a centered multivariate Gaussian
distribution.
\begin{lemma}[restated Lemma \ref{l:conditional}]\label{l:conditional2}
  For any $\Sigmab\in\R^{d\times d}$, the conditional distribution
  of $\Xb\sim\vskx$ given $\Xb^\top\Xb=\Sigmab$ is the same as the conditional
  distribution of $\X\sim \dxk$ given $\X^\top\X=\Sigmab$.
\end{lemma}
\begin{proof}
  Since we are conditioning on an event which may have probability
  $0$, this requires a careful limiting argument. Let $A$ be any
  measurable event over the random  matrix $\Xb$ and let 
    $C_\Sigmab^\epsilon \defeq \big\{\B\in\R^{d\times d}\,:\, \|\B-\Sigmab\|\leq \epsilon\big\}$
  be an $\epsilon$-neighborhood of $\Sigmab$ w.r.t.~the matrix
  $2$-norm such that $\Pr(\Xb^\top\Xb\!\in\! C_\Sigmab^\epsilon > 0)$.
  We write the probability of $\Xb\in A$ conditioned on
  $\Xb^\top\Xb\in C_\Sigmab^\epsilon$ as:
  \begin{align*}
    \Pr\big(\Xb\!\in\! A\,|&\,\Xb^\top\Xb\!\in\! C_\Sigmab^\epsilon\big) =
    \frac{\Pr\big(\Xb\!\in\! A\,\wedge\,\Xb^\top\Xb\!\in\! C_\Sigmab^\epsilon\big)}
{\Pr\big(\Xb^\top\Xb\!\in\! C_\Sigmab^\epsilon\big)}
    =\frac{\E\big[\one_{[\X\in A]}\one_{[\X^\top\X\in
      C_\Sigmab^\epsilon]}\det(\X^\top\X)\big]}
      {\E\big[\one_{[\X^\top\X\in
      C_\Sigmab^\epsilon]}\det(\X^\top\X)\big]}\\
&\leq      \frac{\E\big[\one_{[\X\in A]}\one_{[\X^\top\X\in
      C_\Sigmab^\epsilon]}\det(\Sigmab)(1+\epsilon)^d\big]}
      {\E\big[\one_{[\X^\top\X\in
                                                    C_\Sigmab^\epsilon]}\det(\Sigmab)(1-\epsilon)^d\big]}
=\frac{\E\big[\one_{[\X\in A]}\one_{[\X^\top\X\in
      C_\Sigmab^\epsilon]}\big]}{\E\big[\one_{[\X^\top\X\in
   C_\Sigmab^\epsilon]}\big]} \bigg(\frac{1+\epsilon}{1-\epsilon}\bigg)^d\\
&=\Pr\big(\X\!\in\! A\,|\,\X^\top\X\!\in\! C_\Sigmab^\epsilon\big)
\bigg(\frac{1+\epsilon}{1-\epsilon}\bigg)^d
\overset{\epsilon\rightarrow 0}{\longrightarrow}
\Pr\big(\X\!\in\! A\,|\,\X^\top\X\!=\!\Sigmab\big).
  \end{align*}
  We can obtain a lower-bound analogous to the above upper-bound,
  namely
$\Pr\big(\X\!\in\! A\,|\,\X^\top\X\!\in\! C_\Sigmab^\epsilon\big)
    \big(\frac{1-\epsilon}{1+\epsilon}\big)^d$, which also converges
    to $\Pr\big(\X\!\in\! A\,|\,\X^\top\X\!=\!\Sigmab\big)$.
  Thus, we conclude that:
  \begin{align*}
    \Pr\big(\Xb\!\in\! A\,|\,\Xb^\top\Xb\!=\!\Sigmab\big)
    &=
    \lim_{\epsilon\rightarrow 0}\,\Pr\big(\Xb\!\in\!
A\,|\,\Xb^\top\Xb\!\in\! C_\Sigmab^\epsilon\big)
=\Pr\big(\X\!\in\! A\,|\,\X^\top\X\!=\!\Sigmab\big),
  \end{align*}
  completing the proof.
\end{proof}

\bibliography{pap}
\end{document}